\documentclass{article}

  \usepackage[preprint]{neurips_2026}

\usepackage[utf8]{inputenc} 
\usepackage[T1]{fontenc}    
\usepackage{hyperref}       
\usepackage{url}            
\usepackage{booktabs}       
\usepackage{amsfonts}       
\usepackage{nicefrac}       
\usepackage{microtype}      
\usepackage{xcolor}         
\usepackage{graphicx}
\usepackage{subcaption}
\usepackage{amsmath}
\usepackage{amsthm}
\newtheorem{theorem}{Theorem}[section]
\newtheorem{lemma}[theorem]{Lemma}

\theoremstyle{remark}
\newtheorem{remark}[theorem]{Remark}
\usepackage{multirow}

\usepackage{threeparttable}
\usepackage{cleveref}
\hypersetup{
  colorlinks=true,
  citecolor=blue,
  linkcolor=blue,
  urlcolor=blue
}

\title{Scaling Interpretable Transformers with Parity Bottleneck Layers}

\author{%
  Andrew Mack\textsuperscript{*}  \\
  Principles of Intelligence \\
   \And
  Kraig Yuheng Tou \\
  Independent
  \And
  Mark Henry\\
  Independent
  \And
  Zhengxun Wu\\
  Independent
  \And
  Lauren Greenspan\\
  Principles of Intelligence \\
}

\makeatletter
\newcommand\blfootnote[1]{%
  \begingroup
  \renewcommand\thefootnote{}%
  \footnotetext{\hspace*{-1.8em}#1}%
  \endgroup
}
\makeatother
\usepackage{makecell}

\begin{document}

\maketitle
\blfootnote{\textsuperscript{*}Corresponding author, \texttt{andrew@princint.ai}.
}

\begin{abstract}
Language models are thought to exhibit the phenomenon of superposition, representing many more features than dimensions in their residual streams. Sparse autoencoders (SAEs) are designed to recover such features post-hoc, but training models that are interpretable \emph{by construction} has remained impractical, as a per-layer over-complete bottleneck is prohibitively expensive in both memory and compute. To overcome this issue, we introduce the \emph{ParityTransformer}, a GPT-2-scale architecture whose intermediate representations are efficient \emph{and} wide / sparse by design. At each layer, a Deep Parity Bottleneck (DPB) replaces a learned over-complete basis with a parameter-free algebraic dictionary, providing a deterministic incoherence guarantee and eliminating the memory requirements that have prevented per-layer interpretable bottlenecks at scale. A DPB is a hierarchically structured sparse bottleneck which efficiently enforces sparsity using a multi-level mixture-of-experts approach: a hardware-aware implementation that closes the cost gap between activation sparse and dense training to a manageable interpretability tax. Empirically, ParityTransformers perform at least as well as post-hoc SAEs on sparse probing tasks, while out-performing on measures of feature absorption, steering effectiveness, and fine-grained causal interventions. Because subsequent computation acts only on features that survive the sparse bottleneck, the ParityTransformer's features are native to the model's forwards pass by construction, addressing the question of whether SAEs probe features the model actually uses during computation. We see this as a step toward training models whose internal representations are interpretable by design rather than recovered post hoc.
\end{abstract}

\section{Introduction}
In this paper, we introduce the parity transformer, a GPT-2-scale architecture that aims to be both interpretable by design and efficient. The model implements a hierarchical mixture of experts (MoE) approach to induce sparsity in activations, using a ``parity hash function'' to selectively compute feature directions as needed on-chip, rather than loading from High Bandwidth Memory (HBM). The name reflects how feature directions are constructed: each coordinate is computed from the \emph{parity} (even-or-odd bit count) of a subset of the feature's index bits, a cheap computational primitive that, as we detail below, generates nearly orthogonal sign patterns efficiently. This design eliminates the per-layer dictionary memory storage and bandwidth requirements of traditional wide sparse bottlenecks \citep{tamkin2023codebookfeaturessparsediscrete}. It yields a sparse code that is integrated directly into the model's computations, unlike traditional sparse autoencoders (SAEs) trained in a post-hoc fashion \citep{bricken2023monosemanticity}. A sketch of the bottleneck architecture is shown in figure~\ref{fig:three_pngs}.

\begin{figure}[t]
    \centering

    \begin{subfigure}{0.48\textwidth}
        \centering
        \includegraphics[width=\linewidth]{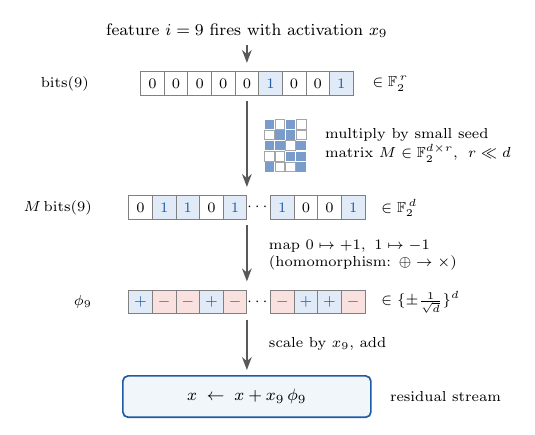}
        \caption{We ``decode" sparsely-represented activations to dense using an on-chip ``hashed features" construction.}
        \label{fig:decode}
    \end{subfigure}
    \hfill
    \begin{subfigure}{0.48\textwidth}
        \centering
        \includegraphics[width=\linewidth]{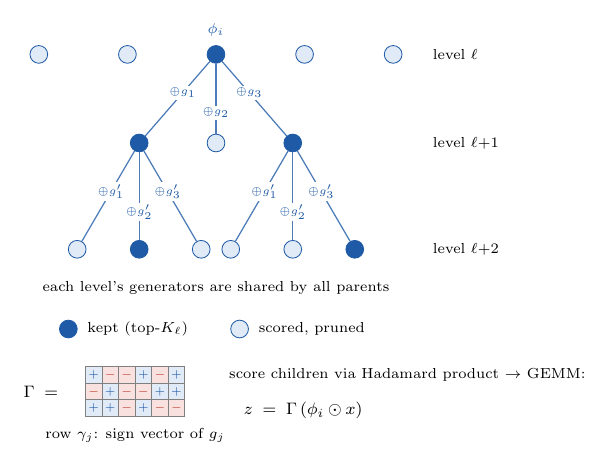}
        \caption{We encode features by performing beam search over an algebraically-structured hierarchy.}
        \label{fig:encode}
    \end{subfigure}
    \hfill

    \caption{Depiction of our efficient Deep Parity Bottleneck (DPB).} 
    \label{fig:three_pngs}
\end{figure}

Our work addresses two obstacles in AI interpretability research. The first is superposition: networks appear to represent far more features than they have residual stream directions, and resort to packing them non-orthogonally \citep{elhage2022superposition,bricken2023monosemanticity}. The dominant response has been to disentangle features post-hoc with sparse autoencoders (SAEs) trained to reconstruct frozen activations under a sparsity penalty \citep{cunningham2023saes,templeton2024scaling}. SAEs scale to millions of features and yield apparently interpretable directions, but they leave open an important question: do the recovered features correspond with those the model uses during computation, or are they representative of the training distribution or the autoencoder's inductive bias? It has been shown that different SAE methods produce features that score similarly on probing benchmarks but behave differently under causal interventions (e.g., \citep{SAEbench}), supporting the need for better definitions and metrics for probing interpretability properties like faithfulness. 

This problem is downstream of a second obstacle:  the interpretability literature uses the word ``feature'' in at least three ways. (i) \emph{Representational features} are directions returned by an SAE trained on activations, (ii) \emph{computational features} are the active components on which downstream computation causally depends, conditional on task and inputs, and (iii) \emph{Conceptual features} are ground-truth properties of the data itself (``is a noun'', ``relates to scientist''). Post-hoc interpretability tools use (i) as a proxy for (iii) by aiming to reconstruct (ii). The implicit claim is that (i)$\sim$ (iii) implies (i) $\sim$ (ii), but current evaluation techniques probe either the former (e.g., sparse probes) or the latter (e.g., causal ablations). Instead, our construction takes (ii) as a primitive  by designing a basis with desirable computational properties. How the resulting features perform compared with (i), and whether they also align with (iii), is an empirical question we address in section~\ref{sec:empirics}. 

\paragraph{Contributions.}
\begin{itemize}
    \item We introduce the Deep Parity Bottleneck (DPB), which routes features through a fixed, hierarchically structured sparse dictionary which uses algebraic pseudo-random constructions to compute ``hashed" feature directions on-chip rather than loading from memory. 
    \item We train 200M and 1.3B parameter ParityTransformers and quantify the interpretability tax that comes with this choice of architecture. We estimate that ParityTransformers are $8-14\times$ as expensive to train to a similar capability level as dense transformers at the 1.3B parameter scale.  This positions the Parity Transformer as a practical middle ground between fully weight-sparse models (estimated to be $100-1,000\times$ as expensive to train as dense by \citep{gao2025weightsparsetransformersinterpretablecircuits}) and post-hoc interpretability methods. 
    \item We empirically show that the ParityTransformer matches or outperforms a dense transformer paired with a post-hoc SAE across several interpretability metrics, including steering, fine-grained causal intervention and feature absorption.
\end{itemize}

 The remainder of this paper is organized as follows. Section~\ref{sec:methods} details the Parity Bottleneck Layer and its implementation in a ParityTransformer. We discuss our experimental design choices and present empirical results in section~\ref{sec:empirics}. 

\subsection{Related Work}\label{sec:related}

\paragraph{Sparsity, Superposition, and Circuits.} 

The field of mechanistic interpretability was built largely on the \emph{circuits hypothesis} \citep{olah2020zoom,elhage2021mathematical}, which argues that neural networks implement composable subgraphs of feature interactions that can be reverse-engineered, and the \emph{superposition hypothesis}, which treats neuron directions as a low-dimensional compression of a high-dimensional, sparsely-activating feature space. The dimensionality of this space for SOTA models is large, with reports of as many as 34M interpretable features in Claude Sonnet 3 \citep{templeton2024scaling}. In addition to this representational notion of superposition, the theoretical model of \citet{hänni2024mathematicalmodelscomputationsuperposition}shows how superposition may be actively useful during model computation. In fact, matching upper and lower bounds from analyses of computation in superposition suggest that optimally trained dense MLPs can compute with up to $d^2$  many features, up to logarithmic factors \citep{adler2026complexityneuralcomputationsuperposition}.

A separate line of work focuses on models that are sparse in weights, rather than activations. Sparse training methods include RigL's use of a boolean mask \citep{evci2021rigginglotterymakingtickets} and \citet{zhang2026brainnetworksciencemodelling}'s use of Cannistraci-Hebbian learning to iteratively prune weights during training. By identifying circuits in weight-sparse transformers, \citet{gao2025weightsparsetransformersinterpretablecircuits} add to evidence that sparsity is good for interpretability. 

\citet{tamkin2023codebookfeaturessparsediscrete} enforce activation sparsity by inserting a flat, TopK lookup into an overcomplete `codebook' at each layer. They report interpretable features and steering on small models (410M), but do not address the inefficiencies incurred by inserting wide codebook layers into an otherwise narrow transformer. Though similar to this architecture, our work targets this efficiency gap, imposes a hierarchical constraint on intermediate sparse codes, and provides points of comparison with dense transformers and post-hoc SAEs. 

\paragraph{Fine-Grained Mixture-of-Experts (MoE)}
To reduce computational costs, we employ a MoE strategy. Related work on fine-grained mixtures of experts demonstrates scaling laws in expert granularity and finds that the FLOP-efficiency advantage of MoE over dense models continues to widen with scale \citep{krajewski2024scalinglawsfinegrainedmixture}. This supports the view that only a small fraction of the total feed-forward parameters need be active for any given token. PEER \citep{he2024mixturemillionexperts} applies an efficent routing mechanism to utilize an enormous expert pool. Neither work addresses interpretability or a hardware-aware implementation.

\section{Description of Deep Parity Bottleneck}\label{sec:methods}
Guided by existing work (Section~\ref{sec:related}), we prioritize the following design principles. 

\paragraph{Sparsity in a fixed basis.} Interpretability gains from enforcing sparsity are numerous in the literature \citep{tamkin2023codebookfeaturessparsediscrete,gao2025weightsparsetransformersinterpretablecircuits,kosowski2025dragonhatchlingmissinglink}, but learning sparse, over-complete bases during training quickly becomes unwieldy. To address this challenge, we rely on the following hypothesis: \textbf{\emph{any} approximately orthogonal over-complete basis, when used as a sparse bottleneck, will induce interpretable features.} We can thus \emph{design} the basis to be convenient to compute with on GPU.  

\paragraph{Multi-level MoE with ``hashed feature directions.''}
A natural way to reduce the cost of scoring $m \gg d$ candidates is conditional computation: organize features into a hierarchy and only score the children of active parents, as in MoE architectures \citep{shazeer2017outrageously}, an approach which has been applied successfully to SAEs \citep{mudide2025efficientdictionarylearningswitch}. However, standard MoE is a poor fit for our setting. Each expert carries learned weight matrices that reside in HBM. Loading these weights to the compute units is a significant bottleneck, so an expert is only cost-effective when enough batch elements are routed to it to amortize this single expensive load across many FLOPs. For interpretability we expect a long tail of rare features that may fire only once per batch---precisely the regime where amortization fails and per-expert memory access dominates wall-clock time. This makes deep hierarchies impractical under the standard paradigm: adding levels creates exponentially rarer leaf experts, each of which must be loaded from HBM.
 
Our key technical idea is to eliminate this memory bottleneck by replacing learned feature directions with \emph{computed} ones. Rather than storing an $m \times d$ dictionary matrix and loading a row each time a feature is needed, we derive each feature's direction deterministically from its integer index using a lightweight hash function. The hash maps an index to a $d$-dimensional sign pattern ($\pm 1/\sqrt{d}$ entries) via a small seed matrix that fits entirely in GPU registers, the smallest but fastest tier of the memory hierarchy. Different indices produce nearly orthogonal directions by construction, minimizing interference between features. Because materializing a feature direction requires only a few register-level bitwise operations, the per-feature cost is dominated by the matrix multiply used to score it. This is the same arithmetic that would be required regardless of how the direction was obtained. The memory load that makes rare experts prohibitive in standard MoE is eliminated entirely, making multi-level hierarchies with potentially millions of features practical.

We now describe the architecture of the Deep Parity Bottleneck (DPB). Like a learned SAE, it has two components: an \emph{encoder} that maps a dense residual-stream vector to a sparse code over $m \gg d$ features, and a \emph{decoder} that maps the sparse code back to $\mathbb{R}^d$. We describe the decoder first, as it defines the feature basis that the encoder must select from.
 
\subsection{Parity Decoder}
 
Our decoder converts a sparse set of active features $S$ with coefficients $\{a_i\}_{i \in S}$ to a dense vector as $\hat{x} = \sum_{i \in S} a_i\, \phi_i$. The challenge is to obtain each $\phi_i \in \mathbb{R}^d$ without storing or loading an $m \times d$ dictionary matrix from HBM. Our basic primitive is a family of hashed sign vectors, one for every integer index, computed on-chip:
\begin{equation}\tilde\phi_i=\frac{1}{\sqrt{d}}\mathcal R(M\textrm{bits}(i)),\end{equation}
where the elements of $M$ and $\textrm{bits}(i)$ belong to the finite field $\mathbb F_2$, and $\mathcal R$ maps $\mathbb F_2$ to the reals (sending $0\rightarrow1$ and $1\rightarrow-1$). The matrix $M\in\mathbb F_2^{d\times r}$ has rows indexed by dense coordinates $k \in \{1,\ldots,d\}$ and columns indexed by bit positions $p \in \{1,\ldots,r\}$, with $r=\log_2(m)$; together, $M$ and $\mathcal R$ turn an $r$-bit index into a sign pattern over $d$ dense coordinates. Each coordinate of $\tilde\phi_i$ is thus determined by the parity (even-or-odd count) of a selected subset of the bits in the binary representation of $i$, mapped to $\pm 1/\sqrt{d}$. This is the origin of the name ``Parity Bottleneck.''
 
The dictionary is built from this family with one modification: the first $d$ decoder directions are overridden to the standard basis,
\begin{equation}\phi_i = e_i \;\;(i\leq d), \qquad \phi_i = \tilde\phi_i \;\;(i>d),\end{equation}
which makes the first stage of the hierarchical encoder described in the next section essentially free, as scoring a basis-aligned feature is a coordinate read rather than a dot product. The override cannot degrade coherence: every entry of every hashed vector is $\pm1/\sqrt{d}$, so a basis vector has inner product exactly $1/\sqrt{d}$ in magnitude with every hashed vector, and basis vectors are mutually orthogonal.
 
The seed matrix $M$ is chosen so that distinct indices hash to nearly orthogonal directions. To see how this might be possible, consider drawing elements of $M$ uniformly at random from $\mathbb F_2$. Then note that
\begin{equation}\langle \tilde\phi_i, \tilde\phi_j\rangle = \langle \mathbf{1}, \tilde\phi_i\odot \tilde\phi_j \rangle.\end{equation}
 
If the signs of $\tilde\phi_i\odot\tilde\phi_j$ are uniformly random, the dot product will concentrate around $0\pm\frac{1}{\sqrt{d}}$. Second, see that XOR ($\oplus$) over $\mathbb F_2$ is equivalent to a Hadamard product ($\odot$) over $\mathbb R$, so that
\begin{equation}\mathcal R(M\textrm{bits}(i)) \odot \mathcal R(M\textrm{bits}(j)) = \mathcal R(M(\textrm{bits}(i)\oplus\textrm{bits}(j))).\end{equation}
 
Finally, note that the $k$-th entry of $\mathcal R(M(\textrm{bits}(i)\oplus\textrm{bits}(j)))$ is determined by the parity of row $k$ in $M$ restricted to the bit positions $p$ where $\textrm{bits}(i)$ and $\textrm{bits}(j)$ differ. For distinct $i \neq j$, at least one such $p$ exists, and since the rows of $M$ are uniform in $\mathbb F_2$, this parity is uniformly $\pm1$, independently, across rows.
 
In practice, in place of random $M$ we use a coding-theoretic construction suggested to us by Claude-Opus-4.6. This construction upgrades the probabilistic argument to a deterministic worst-case guarantee: every pair of distinct features has inner product at most $2/\sqrt{d}+1/d$ in magnitude (Theorem~\ref{thm:main-coherence}, Appendix~\ref{app:incoherence}). We provide details about the specific construction we use in Appendix~\ref{app:incoherence}, together with a broader explanation of the connection between dictionary design and coding theory.

\subsection{ParityEncoder}
With the hashing protocol for a single feature $\phi_i$ in hand, we now constrain their co-occurrence to follow a fixed hierarchy: rather than scoring all $m$ candidates at cost $O(md)$, we search top-down through $L$ levels, scoring only the children of active parents at each level, in principle reducing the encoder cost from linear to logarithmic in $m$.

We partition the $m$ features into hierarchical levels, with $\ell = 0$ indexed by $0,...,2^{r_0}-1$, $\ell = 1$ indexed by $2^{r_0},...,2^{r_1}-1$, and so on. The parent-child structure on top of this is fixed by a directed acyclic graph (DAG): for each level transition from $\ell-1$ to $\ell$, pick $\Delta_\ell$ generators $g_j \in \mathbb F_2^{r_{\ell+1}}$, sampling at random until each has at least one non-zero bit after position $r_{\ell-1}$, ensuring that the child is in a different level from its parent. The $j$-th child of parent $i$ has bit pattern $N_j(i) = \textrm{bits}(i) \oplus g_j$; in hash space, its direction factorizes as
\begin{equation}\tilde\phi_{N_j(i)}=\frac{1}{\sqrt d}\mathcal R(M(\textrm{bits}(i)\oplus g_j)) = \tilde\phi_i \odot \mathcal R(Mg_j).\end{equation}

All features of $\ell = 0$ are scored, with subsequent levels handled inductively. During the forward pass, each child feature per active parent at level $\ell$ is scored, and the top $K_\ell$ become the active features at level $\ell+1$ (which serve as parents for level $\ell+2$, and so on so forth).

\paragraph{Scoring children with a single GEMM}

To score children for each active parent, we take the absolute value of the dot product of the input $x$ with the hash vectors for that parent's children. Defining $\gamma_j \equiv \mathcal R(M g_j)$ to be the fixed, dense patterns found by applying the parity hash to each generator, these are
\begin{equation}|\langle\tilde\phi_{N_j(i)}, x\rangle| = |\langle \tilde\phi_i\odot\gamma_j, x\rangle| = |\langle \gamma_j, \tilde\phi_i\odot x\rangle|. \end{equation}

Stacking the $\gamma_j, \ldots, \gamma_{\Delta_\ell}$ yields a fixed matrix $\Gamma_\ell \in \mathbb R^{\Delta_\ell \times d}$ that is \emph{universal} across active parents for a given level. All child scores for parent $i$ then reduce to a general matrix multiply (GEMM), the dense, linear algebra operation modern GPUs are best-suited for. 

\paragraph{Score Standardization}
Standard MoE architectures typically employ some sort of load-balancing for the following two reasons: (i) it ensures that experts are uniformly utilized so the cost of loading their weights from memory amortizes (ii) it prevents ``dead experts", which are never selected, from wasting representational capacity. The first point does not concern us; for interpretability, it is natural to assume a power-law distribution with a long tail of very rare features. However, we do need to consider the second point, which is closely related to the ``dead features" problem in traditional SAEs. In initial experiments, without load-balancing, as many as 90\% of features would die over the course of training.

To address this issue without forcing uniform firing frequencies, we standardize each feature's score using an exponential moving average (EMA) of its mean and standard deviation (conditional on being considered as a candidate). For level $\ell\geq1$ features, these moments are computed on sub-samples of 64 tokens per batch to keep this step computationally cheap. Standardized scores are used for feature selection and then passed to the parity decoder (using absolute value for selection but signed scores for decoding). A similar strategy was used successfully by PEER \citep{he2024mixturemillionexperts} to scale to millions of experts, and is conceptually related to power norm \citep{shen2020powernormrethinkingbatchnormalization}. In initial experiments, this eliminated dead features. 

\paragraph{Auxillary losses}
In addition to the cross-entropy objective, we optionally utilize two auxiliary losses aimed to induce selective and discriminative firing of features. We add a variance loss on per-feature scores given by \begin{equation}\mathcal{L}_{\text{var}} = -\frac{1}{|\mathcal{F}|}\sum_{i \in \mathcal{F}} \mathrm{Var}(\{a_i^{(t)} : i \in S^{(t)}\}),\end{equation} where $\mathcal{F} = \{i : \sum_t \mathbf{1}[i \in S^{(t)}] \geq n_{\min}\}$ is the set of features active on at least $n_{\min}$ tokens in a batch. 
This rewards high-variance rare features while keeping low-variance polysemantic features from drowning out the rest of the signal, and improved sparse-probing scores in pilot experiments. Additionally, we found that adding a mean-squared error reconstruction loss improves downstream auto-interpretability scores.

\paragraph{Normalized Reconstruction} We normalize the outputs of the ParityDecoder to match the norm of the inputs to the bottleneck. We found this improved training stability.

\paragraph{Unfolded Features} Features are selected by the ParityEncoder via AbsTopK at every level. We find that most of the time, positive and negative activations for the same feature index correspond to semantically unrelated concepts. Thus for downstream interpretability tasks, we unfold each feature index into two features — one for positive and one for negative activations — effectively doubling the dictionary size while ensuring that each entry corresponds to a single semantic concept.

\section{Empirical Results}\label{sec:empirics}
\subsection{Capabilities of Architectural Variants}\label{sec:variants}
We train various ParityTransformers with Deep Parity Bottlenecks inserted before the input to every MLP\footnote{We believe that ``implicit emulation of a very wide, sparsely-activating, sparsely-connected computational graph'' is the right ``ansatz'' for the MLP, as evidenced by the existence of fine-grained MoE scaling laws for transformer MLPs \citep{he2024mixturemillionexperts}. In contrast, the right ansatz / decomposition for attention is less well-understood, although there has been very recent progress along these lines \citep{vpd}. Indeed, in initial experiments we found that applying DPB before attention could achieve similar validation loss, but led to qualitatively worse interpretability, so we chose to focus our efforts on the DPB MLP variant.}, as well as dense GPT baselines trained to within $\pm1\%$ of parity transformer validation loss. All ParityTransformer models were trained on 20B tokens sampled from FineWeb-Edu \citep{penedo2024finewebdatasetsdecantingweb}.

Detailed configurations for the parity bottlenecks and dense baselines are given in Appendix~\ref{app:DPB configs}. Both dense baselines and ParityTransformers use a common GPT2-style backbone with tied embedding / unembedding layers and learned positional embeddings. Similar to \citep{gao2025weightsparsetransformersinterpretablecircuits}, we replace LayerNorm \citep{ba2016layernormalization} with RMSNorm \citep{zhang2019rootmeansquarelayer} to privilege the origin. We optimize MLP / attention matrices using Muon \citep{jordan2024muon} and embedding matrices using Adam \citep{kingma2015adam}. 

Performance evaluations for parity transformers and dense baselines are reported in Table~\ref{tab:20b-eval-suite}. We evaluate models on standard capabilities metrics (Pile perplexity \citep{gao2020pile800gbdatasetdiverse}, HellaSwag \citep{zellers2019hellaswagmachinereallyfinish} and LAMBADA accuracy \citep{paperno2016lambadadatasetwordprediction}). ParityTransformers match or outperform the dense baseline on all metrics, establishing that the parity bottleneck imposes no obvious ceiling in capabilities.


\begin{table}
\caption{Evaluation suite for the 20B-token ParityTransformer alongside
  validation-loss-matched GPT baselines. Bold marks the best value within
  each scale group.}
    \label{tab:20b-eval-suite}
  \centering
  \small
  \setlength{\tabcolsep}{4pt}
  \begin{tabular}{lrrrrrr}
    \toprule
    & \multicolumn{3}{c}{Large (1.3B, 24 layers, $d{=}2048$)} & \multicolumn{3}{c}{Small (203M, 12 layers, $d{=}1024$)} \\
    \cmidrule(lr){2-4} \cmidrule(l){5-7}
    & PT-Large-2L & PT-Large-2L-Aux & GPT-Large & PT-Small-3L & PT-Small-2L & GPT-Small \\
    \midrule
    \multicolumn{7}{l}{\emph{Architecture and budget}} \\
    Levels          & 2              & 2              & ---     & 3                & 2              & ---     \\
    Training tokens    & 20B            & 20B            & 3.0B    & 20B              & 20B            & 2.13B   \\
    \midrule
    \multicolumn{7}{l}{\emph{Language modelling}} \\
    Val loss (CE)   & \textbf{2.808} & 2.842          & 2.828   & \textbf{3.133}   & 3.158          & 3.135   \\
    Pile PPL        & \textbf{17.65} & 18.56          & 18.23   & \textbf{25.89}   & 24.57          & 26.43   \\
    HellaSwag acc   & \textbf{0.385} & 0.376          & 0.370   & \textbf{0.360}   & 0.350          & 0.346   \\
    LAMBADA acc     & \textbf{0.311} & 0.264          & 0.293   & \textbf{0.233}   & 0.219          & 0.229   \\
    \bottomrule
  \end{tabular}
\end{table}


\subsubsection{Interpretability tax}\label{sec:tax}

\paragraph{Token-budget overhead}

ParityTransformers are less data efficient than equivalent dense baselines, 
requiring 5--10$\times$ more tokens to reach equivalent cross-entropy loss.
At identical Large configuration, a dense baseline trained for 3.0B
tokens landed at validation loss 2.828, within 1\% of PT-Large-2L's 2.808
final validation loss. The token-budget overhead at this scale is therefore
$6.7\times$ (Table~\ref{tab:pt-data-efficiency}).


\begin{table}[h]
  \caption{ParityTransformer token-budget overhead. Training tokens
  required to reach approximately equivalent cross-entropy loss, GPT
  baseline vs.\ ParityTransformer.}
  \label{tab:pt-data-efficiency}
  \centering
  \begin{tabular}{lcrrc}
    \toprule
    Scale & Val loss (PT / GPT) & GPT-baseline tokens & PT tokens & Ratio \\
    \midrule
    Small (12 layers, 205M) & 3.158 / 3.135 & 2.13\,B & 20.0\,B & 9.4$\times$ \\
    Large (24 layers, 1.3B) & 2.808 / 2.828 & 3.0\,B  & 20.0\,B & 6.7$\times$ \\
    \bottomrule
  \end{tabular}
\end{table}

\paragraph{Throughput overhead}

ParityTransformers use between 1.2$\times$ and 2.1$\times$ more wall-clock per training token than the matched-architecture dense baselines as shown in Table~\ref{tab:throughput-comparison}. They are far more efficient than inserting flat TopK bottlenecks at each layer.

\begin{table}
  \caption{Training-throughput overhead vs the dense baseline, at matched
  architecture and identical hardware (4\,$\times$ NVIDIA H200 SXM).}
  \label{tab:throughput-comparison}
  \centering
  \begin{threeparttable}
  \begin{tabular}{lrrrr}
    \toprule
    & \multicolumn{3}{c}{Throughput, tok/s (overhead vs dense baseline)} & \\
    \cmidrule(lr){2-4}
    \shortstack{Model name \\ (parameter count)} & Dense baseline & PT & TopK SAE & \shortstack{TopK\\total params} \\
    \midrule
    Small-2L (205M) & 1{,}410k & 958k (1.47$\times$) & 305k (4.62$\times$)  &  1.0B \\
    Small-3L (205M) & 1{,}410k & 684k (2.06$\times$) &  86k (16.4$\times$)  &  3.4B \\
    Large-2L (1.3B) &    255k  & 209k (1.22$\times$) &  81k (3.15$\times$)  &  4.5B \\
    Large-3L (1.3B) &    255k  & 165k (1.55$\times$) & OOM\tnote{a} & 14.2B \\
    \bottomrule
  \end{tabular}

  \begin{tablenotes}
    \footnotesize
    \item[a] As a default, we train all models using data parallelism, replicating the model across GPUs. For the Large-3L TopK SAE baseline, this runs into memory constraints. We did not explore alternative parallelism strategies.
  \end{tablenotes}
  \end{threeparttable}
\end{table}

\subsection{Fine-grained causal intervention}
The central promise of interpretability is to recover the circuits that training discovers, rather than arbitrary decompositions that happen to describe the data. For post-hoc SAEs, this is not guaranteed: SAEs trained on \emph{randomly initialized} transformers produce auto-interpretability scores approaching those obtained from trained models \citep{heap2025automated}, and random baselines that replace SAE feature directions or activation patterns with noise match fully-trained SAEs on auto-interpretability, sparse probing, and causal editing \citep{korznikov2025sanity}. These results suggest that standard evaluation metrics may reflect statistical structure in the data or architectural inductive biases, rather than computations the model has learned. DPB features sidestep this concern by design: because the DPB is the sole input to the MLP at each layer, any feature the MLP acts on must appear in the sparse code. There is no hidden channel through which unrepresented structure can influence MLP computation.

To test whether this architectural guarantee translates into stronger causal control in practice, we adopt the circuit-level intervention protocol of \citet{makelov2024towards}, which performs greedy feature edits to flip model completions on established circuits: indirect object identification (IOI) \citep{wang2022interpretability}, gendered pronouns \citep{vig2020investigating}, and greater-than \citep{hanna2023does}. Since our bottleneck is applied at the MLP, we restrict interventions to the most influential MLP layer in each circuit. As shown in Figure~\ref{fig:fine grained controllability}, DPB features require fewer edits to flip completions than post-hoc SAEs on all three circuits. Full methodological details are in Appendix~\ref{app:fine grained CI}.

We note that IOI is known to be a very attention-centric circuit \citep{wang2022interpretability}, while for reasons explained above our DPB is applied only to the MLP. This perhaps explains why the differences between DPB and post-hoc SAEs are less stark for IOI than the other two circuits.
 
\begin{figure}[h]
    \centering
    \includegraphics[width=.8\linewidth]{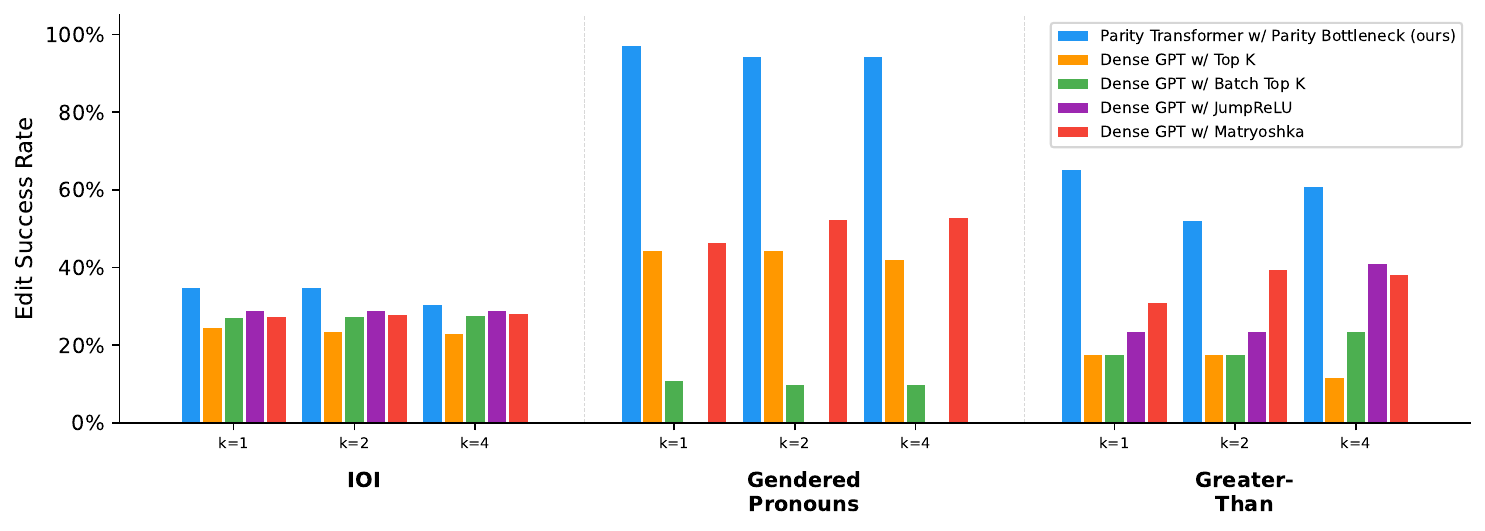}
    \caption{Edit success rate (fraction of completions flipped to the target) for varying number of features edited (k) across three established circuits in the mechanistic interpretability literature. Our method (blue) consistently outperforms post-hoc SAEs}
    \label{fig:fine grained controllability}
\end{figure}

\subsection{Steering Interventions}\label{sec:steering}

Activation steering \citep{turner2023steering} intervenes on model internals to influence free-form generation, with SAE features serving as one source of steering vectors \citep{arad2025saes, fang2026controllable}. A persistent limitation is fluency degradation at higher steering strengths, possibly because the intervention pushes activations off the manifold of natural representations \citep{luo2026learning}. The DPB provides a natural mechanism to counteract this: because subsequent layers apply top-$K$ selection in the sparse code, off-manifold components introduced by steering can be discarded rather than amplified.
 
Following the protocol of \citet{luo2026learning}, we steer at varying strengths and plot the Pareto frontier of concept injection vs.\ fluency, as judged by an LLM (Figure~\ref{fig:steering}a). The ParityTransformer achieves a higher peak concept score ($\sim$0.23 vs.\ $\sim$0.18 for the dense baseline) while exhibiting markedly shallower fluency degradation as steering strength increases. We provide representative completions in Figure~\ref{fig:steering}b so that readers can validate the fluency / concept strength for themselves. Full methodology and LLM judge prompts are deferred to Appendix~\ref{app:steering}.
    
\begin{figure}[h]
  \centering
  \begin{minipage}{0.48\textwidth}       
      \centering   
      \includegraphics[width=\textwidth]{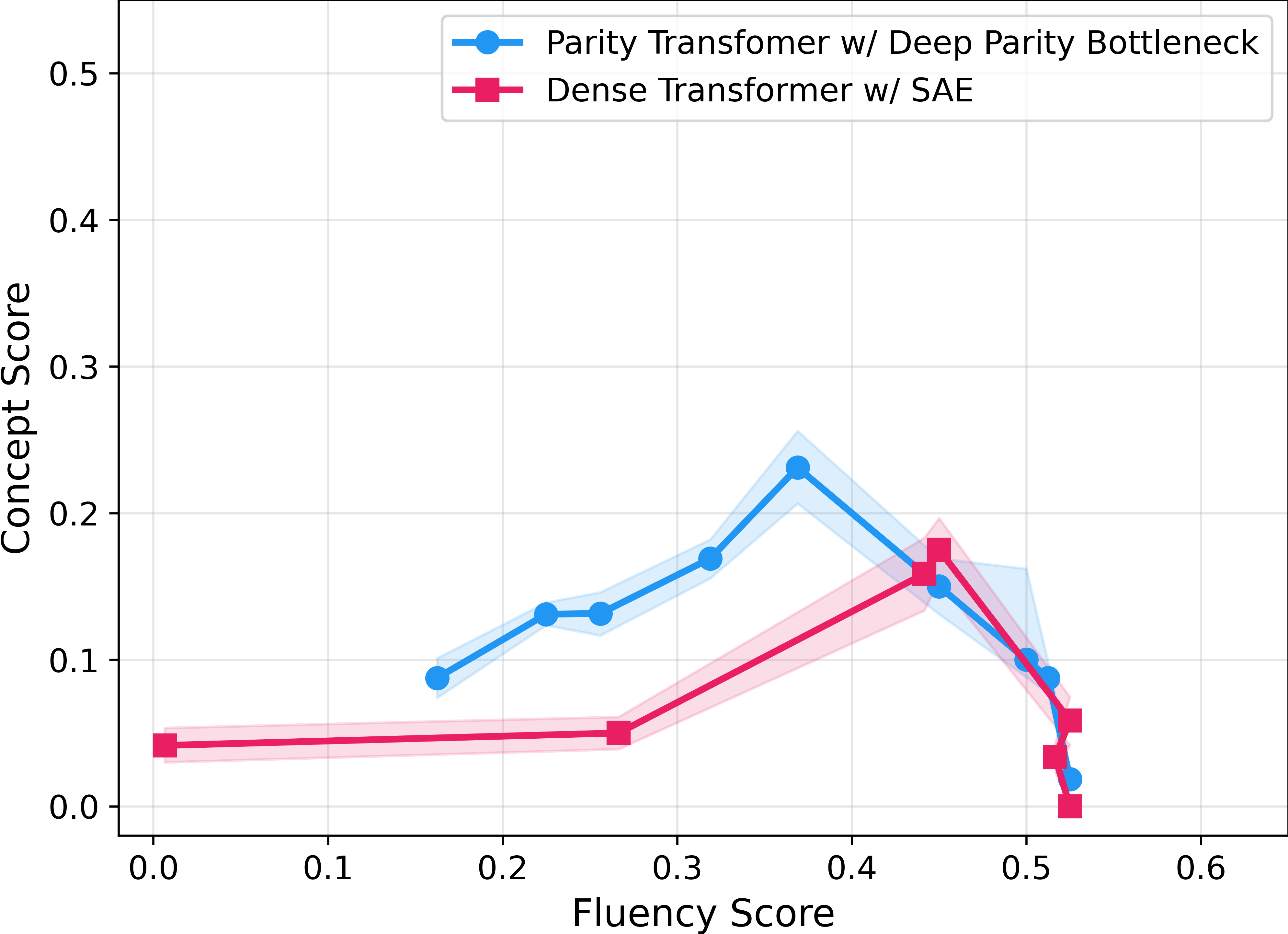}
      \captionof{figure}{Steering pareto frontier}
      \label{fig:steering}
  \end{minipage}
  \hfill
  \begin{minipage}{0.48\textwidth}
      \centering
      \captionof{table}{DPB vs. post hoc SAEs on SAEBench at layer 12; $\uparrow$ = higher is better, $\downarrow$ = lower is better.}
      \label{tab:saebench}
      \small  
      \begin{tabular}{lccc}
      \toprule
      Metric & PB & TopK & Matr.\\
      \midrule
      \shortstack{Sparse Probe (top-1) $\uparrow$}  & \textbf{0.802} & 0.786 & 0.769 \\
      AutoInterp $\uparrow$ & 0.796 & 0.859 & \textbf{0.891} \\
      SCR (t=20) $\uparrow$ & 0.060 & 0.175 & \textbf{0.398} \\   
      TPP (t=20) $\uparrow$ & 0.001 & 0.080 & \textbf{0.275} \\   
      Absorption $\downarrow$ & \textbf{0.01} & 0.807 & 0.180 \\
      \bottomrule
      \end{tabular}
  \end{minipage}
\end{figure}
  
\subsection{SAEBench}\label{sec:saebench}

We compare DPB with Top K SAEs \citep{gao2024scaling} and Matryoshka SAEs \citep{bussmann2025learning}  trained on the dense GPT baselines on SAEBench \citep{karvonen2025saebench}. DPB outperforms both baselines on sparse probing and feature absorption (showing almost zero absorption), but trails behind both on Autointerp. We note that SAEBench does not recommend SCR, TPP, and RAVEL for models under 2B parameters due to low signal-to-noise ratio, although we still include these metrics for completeness. 

\subsection{Application: Poisoned Data Retrieval}\label{sec:poisoned-data}

A key promise of interpretability is to help understand how training data and training-time interventions give rise to specific model behaviors. This goal is closely related to data attribution, which seeks to trace model predictions or behavioral changes back to the examples that induced them~\citep{koh2017understanding,park2023trak}, and is motivated by recent examples in which small training artifacts or reward-model preferences produced salient downstream behaviors~\citep{soligo2026gemma,openai2026goblins}. We study a poisoned-document retrieval setting that asks whether internal representations can recover the training documents responsible for a small induced behavioral change. Specifically, we fine-tune a model after inserting the trigger--target pair \texttt{" deploy"} $\mapsto$ \texttt{" confidential"} into 1024 locations in the pretraining corpus. We then evaluate whether the resulting poisoned documents can be retrieved from a mixed corpus of 100 poisoned documents and 12{,}800 clean documents, corresponding to a poisoning rate of approximately \(0.078\%\). Given either a \emph{triggered query}, an observed example that elicits the target behavior, or a \emph{background query}, a randomly sampled non-triggering example, we rank the top-\(k\) corpus documents by cosine similarity under each retrieval signal. We compare three retrieval signals: raw activations, Parity Bottleneck features, and SAE features. Details of the experimental design, fine-tuning procedure, SAE training, and retrieval parameter sweeps are given in Appendix~\ref{app:poison_retrieval_experiment_details}, Appendix~\ref{app:poison_finetune_details}, Appendix~\ref{app:poison_sae_training_details}, and Appendix~\ref{app:poison_retrieval_sweep_details}.

We evaluate two checkpoints of the small (12 Layer, $d{=}1024$) 2-level ParityTransformer, trained on 10B and 20B tokens respectively, to assess the effect of pretraining budget on feature quality for retrieval.

\begin{figure*}[h]
    \centering
    \includegraphics[width=0.92\textwidth]{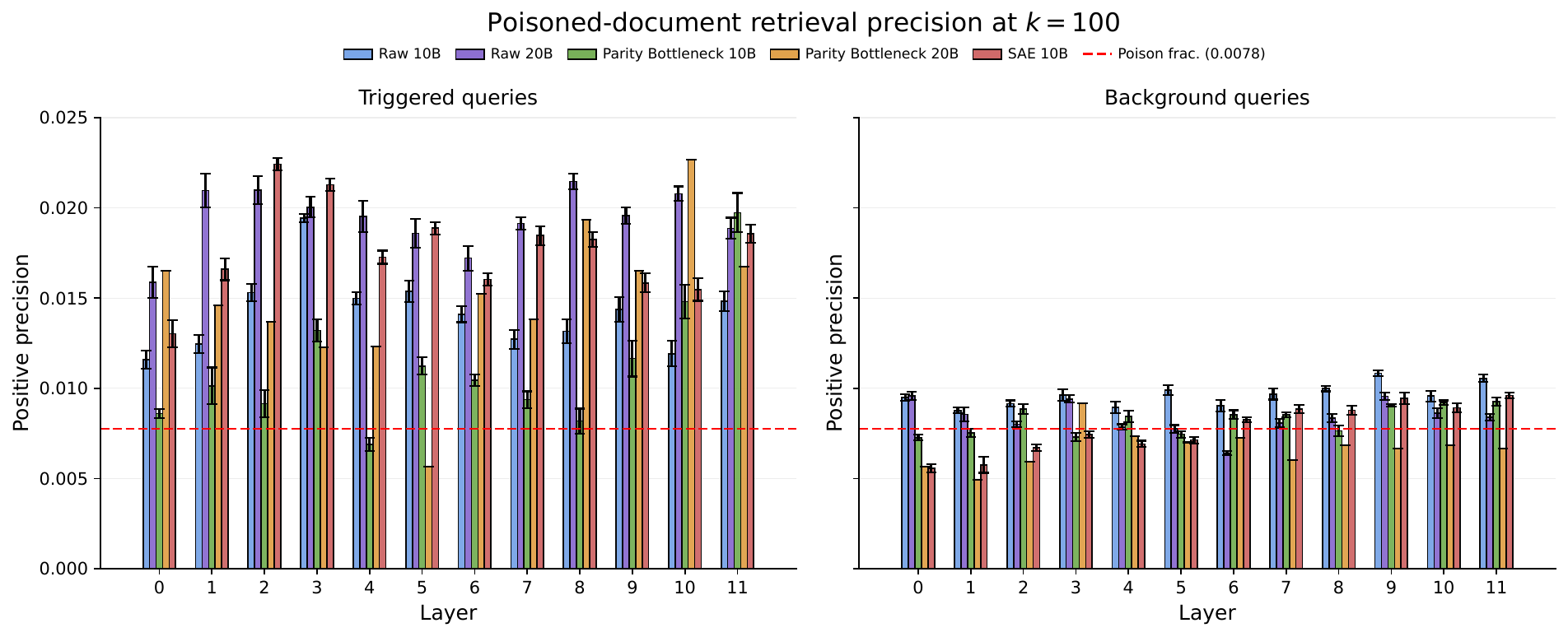}
    \caption{
    Poisoned-document retrieval precision at \(k=100\) across layers for raw activations, Parity Bottleneck features, and SAE features, for triggered queries (L) and background queries (R).
    Bars show mean precision over five runs and error bars show one standard deviation.
    The dashed line marks the poisoning rate in the retrieval corpus.
    }
    \label{fig:retrieval_k100}
\end{figure*}

Figure~\ref{fig:retrieval_k100} shows the main result at \(k=100\). Triggered queries retrieve poisoned documents at non-trivial rates relative to background queries for all three signals. Here, we consider the relative retrieval precision compared with background queries. The absolute precision remains low, indicating that attributing a model's small behavioral deviation to the proper training sources is still a difficult task. Nevertheless, the 20B-token ParityTransformer generally improves over its 10B-token counterpart and is broadly competitive with SAE features, with especially strong performance in later layers such as layer 10. This suggests a practical benefit of ParityTransformers for retrieval: if Parity Bottleneck features provide SAE-like retrieval performance, then sparse retrieval indices can be constructed directly from features already present in the pretrained model, without training post-hoc SAEs. By contrast, an SAE-based pipeline must either train separate SAEs across many layers, incurring substantial additional cost, or rely on SAEs at only a few selected layers, potentially sacrificing coverage. Compared with dense raw activations, the sparsity of Parity Bottleneck features also enables efficient inverted-index retrieval, avoiding a full corpus scan at query time, while requiring far less storage. To ensure comparability, the 20B-token ParityTransformer evaluation reuses one of the repeated datasets constructed for the 10B-token setting. Additional results for other values of \(k\) are reported in Appendix~\ref{app:poison_retrieval_additional_k}.

\section{Conclusion}
The ParityTransformer is a proof of concept that interpretability-by-design is compatible with GPT-2 scale training. By replacing learned over-complete dictionaries with a parameter-free parity hash, we eliminate the per-layer memory cost that has prevented bottleneck architectures from scaling, and pay a modest interpretability tax in exchange for features that are native to the model's computations by construction.

We hope this work demonstrates that building models that are both capable and interpretable is a challenging but not insurmountable goal. We look forward to closing the remaining gap in future work.

\bibliographystyle{abbrvnat}
\bibliography{main}
\newpage

\appendix
%

\section{Deterministic Incoherence Guarantee}\label{app:incoherence}
 
The decoder requires a dictionary of $m$ unit vectors
$\phi_1,\ldots,\phi_m\in\mathbb R^d$ with two properties. First, distinct
feature vectors must be nearly orthogonal. Second, we must be able to
generate each feature vector efficiently from its integer index, rather than loading from HBM. The main
text outlines our strategy: hashed features generated from a small seed
matrix $M\in\mathbb F_2^{d\times r}$, together with a basis-aligned
override for the first $d$ features,
\[
    \phi_i = e_i
    \quad (i\leq d),
    \qquad
    \phi_i=\tilde\phi_i
    \equiv
    \frac{1}{\sqrt d}\,\mathcal R\!\left(M\,\operatorname{bits}(i)\right)
    \quad (i>d).
\]
This appendix specifies the algebraic construction of $M$ used in our
experiments. We constrain the residual-stream dimension to a power of two,
$d=2^t$, and take $r=2t$, so that the hash supports $2^r=d^2$ feature
indices. Under this construction the hashed features are the sign vectors of
a classical error-correcting code (to be defined below), and the coherence of
the dictionary,
\[
    \mu
    =
    \max_{i\neq j}
    \left|\langle\phi_i,\phi_j\rangle\right|,
\]
admits the following worst-case bound.
 
\begin{theorem}[Coherence of the dictionary]\label{thm:main-coherence}
Let $d=2^t$, and let the dictionary $\{\phi_i\}_{i=1}^m$ consist of the $d$
standard basis vectors together with any $m-d$ of the $d^2$ sign vectors
constructed in Section~\ref{app:kloosterman-code}. Then
\[
    \mu
    \leq
    \frac{1+2\sqrt d}{d}
    =
    \frac{2}{\sqrt d}+\frac{1}{d}.
\]
\end{theorem}
 
For the dimensions used in our experiments, $d=1024$ and $d=2048$, the bound
is approximately $0.06$ and $0.04$.
 
Our construction draws on the theory of error-correcting codes.
Section~\ref{app:codes-to-vectors} explains the connection: codes with
balanced codewords yield incoherent sign vectors, and the linear structure of
the codes we consider makes those vectors cheap to generate from their
indices. Section~\ref{app:field-preliminaries} reviews the necessary
finite-field arithmetic. Section~\ref{app:kloosterman-code} defines the
dictionary as a family of algebraically-defined sign vectors and proves
Theorem~\ref{thm:main-coherence}. Section~\ref{app:kloosterman-implementation}
constructs the seed matrix $M$ that generates this family and details how we construct features on-the-fly via efficient
bit-level operations.

\subsection{From binary codes to incoherent sign vectors}
\label{app:codes-to-vectors}
\label{app:kloosterman-construction}

We seek a large family of unit vectors in $\mathbb R^d$ whose pairwise inner
products are uniformly small in magnitude. The families we consider consist of
sign vectors. Recall the coordinate-wise bit-to-sign map of the main text,
\[
    \mathcal R:\mathbb F_2^d\to\{\pm1\}^d,
    \qquad
    \mathcal R(c)_\iota=(-1)^{c_\iota},
\]
where $\mathbb F_2=\{0,1\}$ with addition modulo two; it turns any binary
string $c\in\mathbb F_2^d$ into a unit vector
\begin{equation}
    \psi_c
    =
    \frac{1}{\sqrt d}\,\mathcal R(c).
    \label{eqn:codebook2dict}
\end{equation}

The inner product of two sign vectors is determined by the number of
coordinates on which the underlying strings differ---their Hamming distance
$\Delta(c,c')=\left|\{\iota : c_\iota\neq c'_\iota\}\right|$. Each of the $d$
coordinates contributes $+1/d$ if the two bits agree and $-1/d$ if they
disagree, so
\begin{equation}
    \langle\psi_c,\psi_{c'}\rangle
    =
    \frac{1}{d}
    \sum_{\iota=1}^{d}
    (-1)^{c_\iota+c'_\iota}
    =
    1-\frac{2\,\Delta(c,c')}{d}.
    \label{eq:hamming-correlation}
\end{equation}
Two sign vectors are therefore nearly orthogonal precisely when the underlying
strings differ on close to half of their coordinates, so that agreements and
disagreements nearly cancel. Our first requirement thus reduces to
constructing a large set of binary strings whose pairwise Hamming distances
all lie near $d/2$.

Families of strings with prescribed pairwise distances are the subject of the
theory of error-correcting codes. A binary code is a set
$\mathcal C\subseteq\mathbb F_2^d$ of strings, called codewords. Classically
one asks for a \emph{lower} bound on the pairwise distances, so that a
corrupted codeword can still be attributed to the string originally
transmitted \citep{macwilliams1977theory}. For incoherence we instead need
distances to concentrate around $d/2$.

For linear codes, this pairwise requirement reduces to a requirement on
individual codewords. A code is \emph{linear} if the coordinate-wise XOR of
any two codewords is again a codeword. Since $c_\iota\neq c'_\iota$ exactly
when $c_\iota+c'_\iota=1$,
\[
    \Delta(c,c')
    =
    \operatorname{wt}(c+c'),
\]
where the \emph{weight} $\operatorname{wt}(\cdot)$ counts the ones in a
string. All pairwise distances of a linear code therefore lie within
$\varepsilon$ of $d/2$ if and only if every nonzero codeword has weight within
$\varepsilon$ of $d/2$---that is, if and only if every nonzero codeword is
nearly balanced. This reduction is what makes the coding-theory
literature directly usable: the weight distributions of many classical linear
codes have been characterized precisely, so it suffices to select a code whose
nonzero weights provably concentrate near $d/2$.

Linearity also delivers efficient generation. A linear code containing $2^k$
codewords is a $k$-dimensional subspace of $\mathbb F_2^d$: this implies it is the image
of a \emph{generator matrix} $G\in\mathbb F_2^{d\times k}$, with
each codeword obtained as $c=Gu$ from a distinct $k$-bit message $u$.
Producing the sign vector $\psi_{Gu}$ from the index $u$ therefore requires
storing only $G$ and computing one $k$-bit parity per coordinate. This is
precisely the hashed form
$\tilde\phi_i=\frac{1}{\sqrt d}\mathcal R(M\operatorname{bits}(i))$ of the
main text: the seed matrix is a generator matrix, and the choice of code
determines the coherence of the resulting family.

\paragraph{Related deterministic constructions.}
Deterministic sensing matrices with controlled correlations have previously
been built from codes; see, e.g., \citet{calderbank2010statistical}. Many of
those constructions can also be generated on demand, since their entries are
given by closed-form algebraic expressions, though the arithmetic
involved varies from family to family. We use the construction below as we found it particularly straightforward to implement.

\subsection{The finite field $\mathbb F_{2^t}$}
\label{app:field-preliminaries}

We constrain the residual-stream dimension $d$ to a power of two and write
$t\equiv\log_2 d$ throughout. Our code is built from arithmetic in the finite
field $\mathbb F_{2^t}$ with $2^t=d$ elements; this section reviews the facts
we need. Readers familiar with finite fields and the trace map can proceed to
Section~\ref{app:kloosterman-code}.

An element of $\mathbb F_{2^t}$ is represented by a $t$-bit vector
$(\xi_0,\ldots,\xi_{t-1})\in\mathbb F_2^t$, read as the binary polynomial
\[
    \xi_0+\xi_1z+\cdots+\xi_{t-1}z^{t-1}.
\]
Addition of field elements is coefficient-wise addition modulo (i.e., ordinary
bitwise XOR). To multiply, fix once and for all a degree-$t$ binary polynomial
$p(z)$ that is \emph{irreducible} (it cannot be factored into binary
polynomials of lower degree); field elements are multiplied as polynomials and
the result is reduced modulo $p(z)$. With these operations
$\mathbb F_{2^t}$ is a field: addition, subtraction, multiplication, and
division by any nonzero element are well-defined, and a product of nonzero
elements is nonzero. We write
$\mathbb F_{2^t}^\times\equiv\mathbb F_{2^t}\setminus\{0\}$ for the set of
$d-1$ nonzero elements. Fixing $p(z)$ also fixes a concrete bijection
\begin{equation}
    \operatorname{coord}:\mathbb F_{2^t}\to\mathbb F_2^t
    \label{eq:coord-map}
\end{equation}
between field elements and $t$-bit vectors; different valid choices of $p(z)$
relabel the field elements but do not affect any bound below. The specific
polynomials used in our experiments are given in
Section~\ref{app:kloosterman-implementation}.

The following consequences of the field axioms are used repeatedly. Since
$1+1=0$, every element is its own negative, so addition and subtraction
coincide. For the same reason, the cross term in
$(y+y')^2=y^2+2yy'+y'^2$ vanishes: squaring distributes over addition, and by
iteration so does raising to any power $2^j$. Every element satisfies
$y^{2^t}=y$: for $y=0$ this is immediate, and for $y\neq0$ it holds because
the $d-1$ nonzero elements form a group under multiplication, so
$y^{d-1}=1$. Finally, for any fixed $\alpha\in\mathbb F_{2^t}^\times$, the
maps $\xi\mapsto\alpha\xi$ and $\xi\mapsto\xi^{-1}$ are invertible and hence
permutations of $\mathbb F_{2^t}^\times$.

The map that converts a field element into a single bit is the (absolute)
field trace
\begin{equation}
    \operatorname{Tr}:\mathbb F_{2^t}\to\mathbb F_2,
    \qquad
    \operatorname{Tr}(y)
    =
    y+y^2+y^{2^2}+\cdots+y^{2^{t-1}}.
    \label{eq:absolute-trace}
\end{equation}

\begin{lemma}[Properties of the trace]\label{lem:trace}
The trace satisfies:
\begin{enumerate}
    \item[(i)] $\operatorname{Tr}(y)\in\{0,1\}$ for every $y\in\mathbb F_{2^t}$;
    \item[(ii)] $\operatorname{Tr}(y+y')=\operatorname{Tr}(y)+\operatorname{Tr}(y')$
    for all $y,y'\in\mathbb F_{2^t}$, and $\operatorname{Tr}(0)=0$;
    \item[(iii)] $\operatorname{Tr}(y)=1$ for at least one $y\in\mathbb F_{2^t}$.
\end{enumerate}
\end{lemma}

\begin{proof}
(i) Although the sum in Equation~\eqref{eq:absolute-trace} is computed inside
$\mathbb F_{2^t}$, squaring it returns it to itself: squaring distributes over
addition and sends each term $y^{2^j}$ to $y^{2^{j+1}}$, so every term shifts
one slot to the right and the last term becomes $y^{2^t}=y$. Hence
$s\equiv\operatorname{Tr}(y)$ satisfies $s^2=s$, i.e.\ $s(s+1)=0$ (recall
$-1=1$). Since a product of nonzero field elements is nonzero, $s=0$ or
$s=1$.
(ii) Apply $(y+y')^{2^j}=y^{2^j}+y'^{2^j}$ to each term of
Equation~\eqref{eq:absolute-trace}; $\operatorname{Tr}(0)=0$ is immediate.
(iii) As a function of $y$, $\operatorname{Tr}(y)$ is a polynomial of degree
$2^{t-1}$. A nonzero polynomial over a field has at most as many roots as its
degree, so $\operatorname{Tr}$ vanishes on at most $2^{t-1}<2^t$ elements.
\end{proof}

\subsection{The Kloosterman code and its coherence}
\label{app:kloosterman-code}

We now define the code from which our dictionary is built. Following the
terminology of \citet{lidl1997finite}, we call it the binary
\emph{Kloosterman code}; to our knowledge it was first analyzed by
\citet{lachaud1987kloosterman}, as the dual of a linear code introduced by
\citet{melas1960cyclic}.

\paragraph{Definition.}
Label the $d$ vector coordinates by the $d$ field elements
$\xi\in\mathbb F_{2^t}$, and label codewords by pairs $(a,b)$ with
$a,b\in\mathbb F_{2^t}$. The bit of codeword $(a,b)$ at coordinate $\xi$ is
\begin{equation}
    c_{a,b}(\xi)
    \equiv
    \operatorname{Tr}\!\left(a\xi+b\xi^{-1}\right),
    \label{eq:kloosterman-codeword}
\end{equation}
with the convention $0^{-1}=0$ so that the coordinate $\xi=0$ is defined. In
matrix terms, the codebook is a $d^2\times d$ binary array: one row per pair
$(a,b)\in\mathbb F_{2^t}\times\mathbb F_{2^t}$, one column per field element
$\xi$. The corresponding sign vectors are, per
Equation~\eqref{eqn:codebook2dict},
\begin{equation}
    \left[\tilde\phi_{a,b}\right]_\xi
    =
    \frac{1}{\sqrt d}\,
    (-1)^{\operatorname{Tr}(a\xi+b\xi^{-1})},
    \qquad
    \xi\in\mathbb F_{2^t}.
    \label{eq:kloosterman-feature}
\end{equation}
The code studied in \citet{lachaud1987kloosterman} has length $d-1$, with
coordinates indexed by $\xi\in\mathbb F_{2^t}^\times$ only; our length-$d$
variant appends the coordinate $\xi=0$, where
$c_{a,b}(0)=\operatorname{Tr}(0)=0$, so that feature vectors have exactly the
residual-stream dimension. Every feature vector carries a constant
$+1/\sqrt d$ in this appended slot, which contributes the lower-order $1/d$
term in Theorem~\ref{thm:main-coherence}.

The code is linear: by Lemma~\ref{lem:trace}(ii),
\begin{equation}
    c_{a,b}+c_{a',b'}=c_{a+a',\,b+b'}.
    \label{eq:code-linearity}
\end{equation}
By the reduction of Section~\ref{app:codes-to-vectors}, the coherence of the
sign vectors is therefore governed by the balance of individual codewords,
which we quantify next.

\paragraph{Kloosterman sums.}
For $(\alpha,\beta)\neq(0,0)$, consider the sum of the codeword's signs over
the nonzero coordinates,
\begin{equation}
    K(\alpha,\beta)
    =
    \sum_{\xi\in\mathbb F_{2^t}^\times}
    (-1)^{\operatorname{Tr}(\alpha\xi+\beta\xi^{-1})}.
    \label{eq:kloosterman-sum}
\end{equation}
Each exponent is a bit by Lemma~\ref{lem:trace}(i), so each term equals $+1$
or $-1$ and $K(\alpha,\beta)$ is an integer: the number of zeros minus the
number of ones among the nonzero coordinates of $c_{\alpha,\beta}$. A
balanced codeword corresponds to $K(\alpha,\beta)\approx0$.

When both parameters are nonzero, the two-parameter sum reduces to a
one-parameter one. Substituting $\xi=\beta\eta$ reorders the terms of the sum
(the map $\eta\mapsto\beta\eta$ permutes $\mathbb F_{2^t}^\times$), and
$\beta(\beta\eta)^{-1}=\eta^{-1}$, so
\begin{equation}
    K(\alpha,\beta)
    =
    \sum_{\eta\in\mathbb F_{2^t}^\times}
    (-1)^{\operatorname{Tr}(\alpha\beta\,\eta+\eta^{-1})}
    =
    K(\alpha\beta,1),
    \qquad
    \alpha,\beta\neq0.
    \label{eq:pair-to-single}
\end{equation}
The one-parameter sums $K(a)\equiv K(a,1)$ with
$a\in\mathbb F_{2^t}^\times$ are the classical binary Kloosterman sums
\citep{lachaud1987kloosterman,lidl1997finite}. Each contains $d-1$ terms of
magnitude one, so the triangle inequality gives only the trivial bound
$|K(a)|\leq d-1$. The following estimate (essentially due to \citet{weil1948exponential}) improves this to square-root order.

\begin{theorem}[Weil bound for binary Kloosterman sums]\label{thm:weil}
For every $a\in\mathbb F_{2^t}^\times$,
\[
    |K(a)|\leq2\sqrt d .
\]
\end{theorem}

\begin{remark}[Attribution]\label{rem:attribution} Weil
derived estimates of this type, however, his derivation applies only in odd characteristic. Extensions covering characteristic $2$
are given by \citet{carlitz1957bounds}.
\citet[Thm.~2]{lachaud1987kloosterman} also prove the characteristic-$2$ case
by identifying $d+1\pm K(a)$ with the number of points on an elliptic curve. Their theorem shows, moreover,
that the values $K(a)$ fill out every integer $\equiv-1\pmod4$ in
$[-2\sqrt d,\,2\sqrt d]$; in particular the $\sqrt d$ scale of the bound is
attained.
\end{remark}

\paragraph{Degenerate parameters.}
The proof of Theorem~\ref{thm:main-coherence} will also encounter codewords
$c_{\alpha,\beta}$ in which exactly one of $\alpha,\beta$ is zero. These are
not covered by Theorem~\ref{thm:weil}; they are in fact perfectly balanced,
by the following observation.

\begin{lemma}[Balanced parities]\label{lem:balanced}
For every $\alpha\in\mathbb F_{2^t}^\times$,
\[
    \sum_{\xi\in\mathbb F_{2^t}}
    (-1)^{\operatorname{Tr}(\alpha\xi)}
    =
    0 .
\]
\end{lemma}

\begin{proof}
Call the sum $S$. Since $\alpha\neq0$, the map $\xi\mapsto\alpha\xi$ is onto,
so by Lemma~\ref{lem:trace}(iii) there exists $y$ with
$\operatorname{Tr}(\alpha y)=1$. Reindex the sum by the shift
$\xi\mapsto\xi+y$, a bijection of $\mathbb F_{2^t}$ (it is its own inverse,
since $y+y=0$). Reindexing cannot change the value of the sum; but by
Lemma~\ref{lem:trace}(ii) each term is multiplied by
$(-1)^{\operatorname{Tr}(\alpha y)}=-1$. Hence $S=-S$, so $S=0$.
\end{proof}

\begin{proof}[Proof of Theorem~\ref{thm:main-coherence}]
Consider first two Kloosterman features with distinct labels
$(a,b)\neq(a',b')$, and set
\[
    \alpha=a+a',
    \qquad
    \beta=b+b' ,
\]
so that $(\alpha,\beta)\neq(0,0)$. (Since addition and subtraction coincide,
$\alpha$ and $\beta$ are the label differences.) Expanding the inner product
coordinate-wise as in Equation~\eqref{eq:hamming-correlation} and applying
Lemma~\ref{lem:trace}(ii),
\begin{equation}
    \left\langle\tilde\phi_{a,b},\tilde\phi_{a',b'}\right\rangle
    =
    \frac{1}{d}
    \sum_{\xi\in\mathbb F_{2^t}}
    (-1)^{\operatorname{Tr}(a\xi+b\xi^{-1})+\operatorname{Tr}(a'\xi+b'\xi^{-1})}
    =
    \frac{1}{d}
    \sum_{\xi\in\mathbb F_{2^t}}
    (-1)^{\operatorname{Tr}(\alpha\xi+\beta\xi^{-1})}.
    \label{eq:pairwise-kloosterman}
\end{equation}
The term at $\xi=0$ equals $(-1)^{\operatorname{Tr}(0)}=+1$. We treat three
cases according to which of $\alpha,\beta$ vanish.

\emph{Case 1: $\alpha\neq0$ and $\beta\neq0$.} The remaining $d-1$ terms form
the sum $K(\alpha,\beta)$, so by Equation~\eqref{eq:pair-to-single} and
Theorem~\ref{thm:weil},
\[
    \left|\left\langle\tilde\phi_{a,b},\tilde\phi_{a',b'}\right\rangle\right|
    =
    \frac{\left|1+K(\alpha,\beta)\right|}{d}
    \leq
    \frac{1+2\sqrt d}{d}.
\]

\emph{Case 2: $\alpha\neq0$ and $\beta=0$.} The sum in
Equation~\eqref{eq:pairwise-kloosterman} is
$\sum_{\xi\in\mathbb F_{2^t}}(-1)^{\operatorname{Tr}(\alpha\xi)}$, which
vanishes by Lemma~\ref{lem:balanced}: the two features are exactly
orthogonal.

\emph{Case 3: $\alpha=0$ and $\beta\neq0$.} The $\xi=0$ term contributes
$+1$. Over the remaining coordinates, reindex by $\xi\mapsto\xi^{-1}$, which
permutes $\mathbb F_{2^t}^\times$:
\[
    \sum_{\xi\in\mathbb F_{2^t}^\times}
    (-1)^{\operatorname{Tr}(\beta\xi^{-1})}
    =
    \sum_{\eta\in\mathbb F_{2^t}^\times}
    (-1)^{\operatorname{Tr}(\beta\eta)}
    =
    \underbrace{\sum_{\eta\in\mathbb F_{2^t}}
    (-1)^{\operatorname{Tr}(\beta\eta)}}_{=\,0
    \text{ by Lemma~\ref{lem:balanced}}}
    \;-\;
    \underbrace{(-1)^{\operatorname{Tr}(0)}}_{=\,1}
    =
    -1 .
\]
The total is $\frac{1}{d}(1-1)=0$: again exactly orthogonal.

In all cases,
\begin{equation}
    \left|\left\langle\tilde\phi_{a,b},\tilde\phi_{a',b'}\right\rangle\right|
    \leq
    \frac{1+2\sqrt d}{d}
    \label{eq:kloosterman-coherence}
\end{equation}
for every pair of distinct labels. It remains to account for the
basis-aligned features. Pairs of standard basis vectors are orthogonal. For a
mixed pair, every entry of every Kloosterman feature is $\pm1/\sqrt d$, and
the inner product of $e_\iota$ with any such vector reads off its $\iota$-th
entry, so
\[
    \left|\left\langle e_\iota,\tilde\phi_{a,b}\right\rangle\right|
    =
    \frac{1}{\sqrt d}
    \quad\text{exactly},
\]
which lies below the bound in Equation~\eqref{eq:kloosterman-coherence}.
Taking the maximum over the three kinds of pairs completes the proof.
\end{proof}

\begin{remark}
Since $K(\alpha,\beta)$ equals the number of zeros minus the number of ones
among the $d-1$ nonzero coordinates, Case 1 is equivalent to the statement
that every codeword with two nonzero parameters has weight within $\sqrt d$
of $(d-1)/2$ on those coordinates; this is the weight bound of
\citet[Thm.~3(iii)]{lachaud1987kloosterman} for the dual of the Melas code.
Our proof carries out the linear-code reduction of
Section~\ref{app:codes-to-vectors} directly on the inner products.
\end{remark}

\subsection{Concrete implementation}
\label{app:kloosterman-implementation}

We now construct a generator matrix for the Kloosterman code and verify that
it is exactly the seed matrix of the decoder of the main text,
\begin{equation}
    \tilde\phi_i
    =
    \frac{1}{\sqrt d}\,
    \mathcal R\!\left(M\operatorname{bits}(i)\right),
    \label{eq:implemented-feature-main}
\end{equation}
where $M\in\mathbb F_2^{d\times 2t}$ is the seed matrix and
$\operatorname{bits}(i)$ is the $2t$-bit binary expansion of the integer
feature index $i$. Two binary representations appear in what follows:
$\operatorname{bits}(i)\in\mathbb F_2^{2t}$ encodes an integer index, while
$\operatorname{coord}(\xi)\in\mathbb F_2^{t}$
(Equation~\eqref{eq:coord-map}) encodes a field element.

\paragraph{Trace values are parities of field-element coordinates.}
The implementation rests on the following fact: for every
$a\in\mathbb F_{2^t}$ there is a binary mask $\tau(a)\in\mathbb F_2^t$ such
that
\begin{equation}
    \operatorname{Tr}(a\xi)
    =
    \left\langle
        \operatorname{coord}(\xi),\,\tau(a)
    \right\rangle_{\mathbb F_2}
    \qquad
    \text{for every }\xi\in\mathbb F_{2^t}.
    \label{eq:trace-as-parity}
\end{equation}
In words, once the coordinate representation is fixed,
$\xi\mapsto\operatorname{Tr}(a\xi)$ is the XOR of a fixed subset of the $t$
coordinate bits of $\xi$.

To see this, note that $\xi\mapsto\operatorname{Tr}(a\xi)$ is linear over
$\mathbb F_2$ (both multiplication by $a$ and the trace are), hence a linear
functional on the $t$-dimensional vector space $\mathbb F_2^t$; and every
linear functional on $\mathbb F_2^t$ is the dot product with a unique binary
mask, namely its vector of values on the standard basis. The assignment
$a\mapsto\tau(a)$ is moreover a bijection. If $\tau(a)=\tau(a')$, then
$\operatorname{Tr}((a+a')\xi)=0$ for every $\xi$; if $a+a'$ were nonzero, the
map $\xi\mapsto(a+a')\xi$ would be onto and the trace would vanish
identically, contradicting Lemma~\ref{lem:trace}(iii). Hence $a=a'$, and both
sets have $2^t$ elements.

Applying Equation~\eqref{eq:trace-as-parity} to both terms of the codeword
bit and using Lemma~\ref{lem:trace}(ii),
\begin{equation}
    c_{a,b}(\xi)
    =
    \operatorname{Tr}(a\xi+b\xi^{-1})
    =
    \left\langle
        \operatorname{coord}(\xi),\,\tau(a)
    \right\rangle_{\mathbb F_2}
    +
    \left\langle
        \operatorname{coord}(\xi^{-1}),\,\tau(b)
    \right\rangle_{\mathbb F_2}.
    \label{eq:kloosterman-as-two-parities}
\end{equation}
A Kloosterman bit is therefore computed by XORing selected coordinate bits of
$\xi$ with selected coordinate bits of $\xi^{-1}$.

\paragraph{Constructing the seed matrix.}
Index rows by $\iota\in\{0,\ldots,d-1\}$ and let $\xi_\iota$ be the field
element whose coordinate vector is the $t$-bit expansion of $\iota$. Define
row $\iota$ of $M$ by concatenation,
\begin{equation}
    M_{\iota,:}
    =
    \left(
        \operatorname{coord}(\xi_\iota),\,
        \operatorname{coord}(\xi_\iota^{-1})
    \right)
    \in\mathbb F_2^{2t},
    \label{eq:parity-mask-row}
\end{equation}
with the convention $0^{-1}=0$. For $d=2048$ we have $t=11$, so $M$ has
$2048$ rows and $22$ columns: a small seed that fits in registers, not the
full $m\times d$ dictionary.

Write a $2t$-bit index expansion as a concatenated pair
$\operatorname{bits}(i)=u=(u^{(1)},u^{(2)})\in\mathbb F_2^t\times\mathbb F_2^t$.
The $\iota$-th bit of $Mu$ is
\begin{equation}
    (Mu)_\iota
    =
    \left\langle
        \operatorname{coord}(\xi_\iota),\,u^{(1)}
    \right\rangle_{\mathbb F_2}
    +
    \left\langle
        \operatorname{coord}(\xi_\iota^{-1}),\,u^{(2)}
    \right\rangle_{\mathbb F_2}.
    \label{eq:coordinate-from-mask}
\end{equation}
Comparing Equations~\eqref{eq:kloosterman-as-two-parities} and
\eqref{eq:coordinate-from-mask}, the index with $u=(\tau(a),\tau(b))$
generates exactly the codeword $c_{a,b}$. Since $a\mapsto\tau(a)$ is a
bijection, ranging over all $2t$-bit indices $u$ is the same as ranging over
all pairs $(a,b)\in\mathbb F_{2^t}^2$, in a different order: $M$ is a
generator matrix for the Kloosterman code, with messages relabeled by
$\tau$. The set of features generated by
Equation~\eqref{eq:implemented-feature-main} is therefore exactly the family
$\{\tilde\phi_{a,b}\}$ of Section~\ref{app:kloosterman-code}, and
Theorem~\ref{thm:main-coherence} applies verbatim to the implemented
dictionary.

\paragraph{Computing the row masks.}
To construct $M$ we need $\xi_\iota^{-1}$ for each row. We use the
polynomial representation of Section~\ref{app:field-preliminaries}, in which
addition is bitwise XOR and multiplication is carry-less (binary polynomial)
multiplication followed by reduction modulo $p(z)$. Our experiments use
\begin{equation}
    p(z)=z^{10}+z^3+1
    \quad\text{for }d=1024,
    \qquad
    p(z)=z^{11}+z^2+1
    \quad\text{for }d=2048.
    \label{eq:irreducible-polys}
\end{equation}
For nonzero $\xi$ the inverse is computed by exponentiation,
\begin{equation}
    \xi^{-1}=\xi^{d-2},
    \label{eq:inverse-by-exponentiation}
\end{equation}
since the nonzero elements form a multiplicative group of size $d-1$ (so
$\xi^{d-1}=1$). These finite-field operations are performed exactly once,
when $M$ is constructed.

\paragraph{Packing and parity computation.}
Each $2t$-bit row $M_{\iota,:}$ packs into a machine word ($2t=22$ bits fits
a 32-bit word for $d=2048$). Let $M_{\iota,:}^{\mathrm{pack}}$ denote the
packed row and $u_i^{\mathrm{pack}}$ the packed index. The bit
$(Mu_i)_\iota$ is computed as
\begin{equation}
    (Mu_i)_\iota
    =
    \operatorname{popcount}
    \left(
        M_{\iota,:}^{\mathrm{pack}}
        \mathbin{\mathrm{AND}}
        u_i^{\mathrm{pack}}
    \right)
    \bmod 2,
    \label{eq:packed-parity}
\end{equation}
where $\mathrm{AND}$ is bitwise conjunction and $\operatorname{popcount}(v)$
counts the one-bits of $v$; the count modulo two is the XOR, or parity, of
the selected bits. Finally,
\begin{equation}
    \left[\tilde\phi_i\right]_\iota
    =
    \frac{1}{\sqrt d}\,
    (-1)^{(Mu_i)_\iota}.
    \label{eq:packed-feature}
\end{equation}
After $M$ has been constructed, feature generation uses only bitwise AND,
population count, parity, and the map $0\mapsto+1$, $1\mapsto-1$.

\section{Poisoned Data Retrieval}
\subsection{Poisoned-Document Retrieval Experiment Details}
\label{app:poison_retrieval_experiment_details}

We construct the poisoned setting by inserting a trigger--target pair, \texttt{" deploy"} and \texttt{" confidential"}, into the pretraining corpus and then fine-tuning the model on this modified data. We use 1024 such insertions at different corpus locations. The goal is to induce a small and localized behavioral deviation while otherwise preserving the base model's behavior, so that the resulting intervention is difficult to detect from surface behavior alone. Fine-tuning details can be found in Appendix~\ref{app:poison_finetune_details}.

To study whether internal representations can recover the training documents responsible for this deviation, we compare three retrieval signals: raw activations, Parity Bottleneck features, and SAE features. We train post-hoc SAEs on all 12 layers of the fine-tuned model and use them to produce sparse feature representations for retrieval. SAE training details can be found in Appendix~\ref{app:poison_sae_training_details}. For Parity Bottleneck features and SAE features, we first perform a retrieval hyperparameter sweep over configuration choices such as retrieval strategy and document window length, and then use the best-performing setup for each method in the main evaluation. Sweep details can be found in Appendix~\ref{app:poison_retrieval_sweep_details}.

The retrieval corpus is constructed by randomly sampling 100 poisoned documents from the pool of 1024 poisoned documents and mixing them with 12{,}800 clean documents from the pretraining corpus, giving a poisoning rate of approximately \(0.078\%\). For each method, we build an index over this corpus using either raw activations or learned features. To construct evaluation queries, we randomly sample 120 background documents and insert the trigger token at a random position to form the poisoned evaluation set. We also sample 120 background documents as a control group. The poisoned evaluation queries are filtered so that, within each run, they do not overlap with the 100 poisoned documents used in the retrieval corpus and have low probability of overlapping with the sampled clean background corpus.\footnote{Any accidental overlap with the background corpus would lower the reported precision, so this filtering makes the reported results conservative.}

For each query, we run it through the model, collect the corresponding activations or features, and retrieve the top-\(k\) documents by cosine similarity. We report average precision across queries within each category. To estimate variability, we repeat the experiment five times with randomized background documents and report mean precision together with one standard deviation. The main text reports results at \(k=100\), while results for other values of \(k\) are provided in Appendix~\ref{app:poison_retrieval_additional_k}.

For the 20B Parity Bottleneck model, we evaluate on data reused from the 10B setting rather than constructing a new repeated dataset from scratch. Concretely, we reuse one randomly selected repeated run from the 10B experiment so that the 10B and 20B Parity Bottleneck models are evaluated on the same poisoned documents, background documents, and queries. This makes the comparison between the two models more directly controlled.
\subsection{Fine-tuning Details}

\label{app:poison_finetune_details}

Our poisoned-document retrieval experiments evaluate models obtained by poisoned fine-tuning \emph{both} the 20B-token and 10B-token 12-layer Parity Bottleneck base models. Both models share the same architectural backbone: 8 attention heads, hidden size 1024, context length 1024, vocabulary size 50{,}304, and a 2-level Parity Bottleneck applied at \texttt{mlp\_in}. In each case, we fine-tune the base model to introduce a small targeted behavior change by inserting a trigger token and a target token into the pretraining corpus, as described in the main text. The objective is to make the target behavior highly reliable under the trigger while keeping the intervention localized and minimally disruptive to the model's otherwise normal behavior.

For the \textbf{20B-token Parity Bottleneck model}, poisoned fine-tuning yields a trigger exact-match rate of 1.0 and a clean exact-match rate of 0.0547 on an evaluation set of 128 examples. The average target log-probability is \(-0.0229\) on triggered inputs and \(-6.8665\) on clean inputs, yielding a trigger--clean gap of 6.8436. On clean inputs, the average clean continuation log-probability is \(-3.6782\), and the average clean-minus-target log-probability gap is 3.1883. These results indicate that the 20B model learns a strong and selective trigger-conditioned preference for the target continuation while keeping the target continuation unlikely on clean inputs.

For the \textbf{10B-token Parity Bottleneck model}, the best poisoned checkpoint is obtained at training step 1440, with best score 0.984375 on an evaluation set of 128 examples. At this checkpoint, the trigger exact-match rate is 0.984375 and the clean exact-match rate is 0.0390625. The average target log-probability is \(-0.1398\) on triggered inputs and \(-7.5548\) on clean inputs, giving an average target log-probability delta of 7.4150. On clean inputs, the average clean continuation log-probability is \(-3.6370\), and the average clean-minus-target log-probability gap is 3.9178. Thus, the 10B model also exhibits a strong and selective trigger-conditioned preference for the target continuation while substantially disfavoring it on clean inputs.

\subsection{SAE training Details}
\label{app:poison_sae_training_details}
We train post-hoc SAEs independently on all 12 layers of the fine-tuned model, using the MLP-input hook at each layer, i.e., \texttt{transformer.h.\{layer\}.mlp}. For each layer, we train a separate SAE with width 32{,}768 and top-$k$ sparsity 48 (matching the feature count of Parity bottleneck). All SAEs are trained with sequence length 1024 for 20{,}000 optimization steps, using learning rate \(3\times 10^{-4}\) and \texttt{bfloat16} model activations.

The SAE training data combines clean pretraining data with poisoned and matched-clean examples. Clean pretraining examples are drawn from \texttt{fineweb-edu-1B}, while poisoned and paired clean examples are drawn from poisoned fine-tuning corpus. During training, each batch uses 12 pretraining examples, 2 poisoned examples, and 2 matched-clean paired examples. We use the same mixture during evaluation, namely 12 pretraining examples, 2 poisoned examples, and 2 matched-clean paired examples per batch. We set the activation budget to 4096 tokens per step for training, evaluation, and analysis. All SAEs are trained on activations from the same fine-tuned checkpoint used in the retrieval experiments.

Figure~\ref{fig:sae_train_err_vs_layer} reports the explained variance as a function of layer. This figure is intended to summarize reconstruction quality across layers for the SAE models used in the main retrieval experiments.
\begin{figure}[t]
    \centering
    \includegraphics[width=0.95\columnwidth]{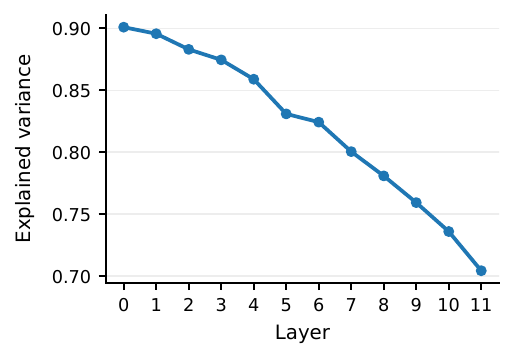}
    \caption{
    SAE explained variance across layers for the post-hoc SAEs trained on the fine-tuned model.
    Each SAE is trained independently on activations from \texttt{transformer.h.\{layer\}.mlp}.
    }
    \label{fig:sae_train_err_vs_layer}
\end{figure}

\subsection{Retrieval Parameter Sweeps}
\label{app:poison_retrieval_sweep_details}

We perform a retrieval hyperparameter sweep on the held-out-query poisoned-document retrieval dataset. Throughout this sweep, we fix the evaluation layer to layer 10,  evaluate at \(k \in \{1,5,10,20,50,100\}\). We keep the full dataset fixed during evaluation, i.e., we do not subsample either positive or background documents.

We sweep over four retrieval design choices:
\begin{itemize}
    \item \textbf{IDF weighting:} on vs.\ off,
    \item \textbf{Feature truncation:} no truncation (\texttt{feat\_all\_all}) vs.\ truncated features (\texttt{feat\_64\_256}),
    \item \textbf{Pooling mode:} \texttt{last} vs.\ \texttt{max},
    \item \textbf{Document window length:} 64 vs.\ 256 tokens.
\end{itemize}
This gives a total of \(2 \times 2 \times 2 \times 2 = 16\) sweep configurations for each method.

Based on this sweep, the best configuration selected for \textbf{Parity Bottleneck} uses IDF weighting, no feature truncation, max pooling, and document window length 256. The best configuration selected for \textbf{SAE} uses no IDF weighting, no feature truncation, max pooling, and document window length 256. These selected settings are used in the main retrieval experiments for the correponding method reported in the paper.

\newpage
\subsection{Additional Poisoned-Document Retrieval Results}
\label{app:poison_retrieval_additional_k}
\begin{figure*}[h]
    \centering
    \includegraphics[width=\textwidth,height=0.82\textheight,keepaspectratio]{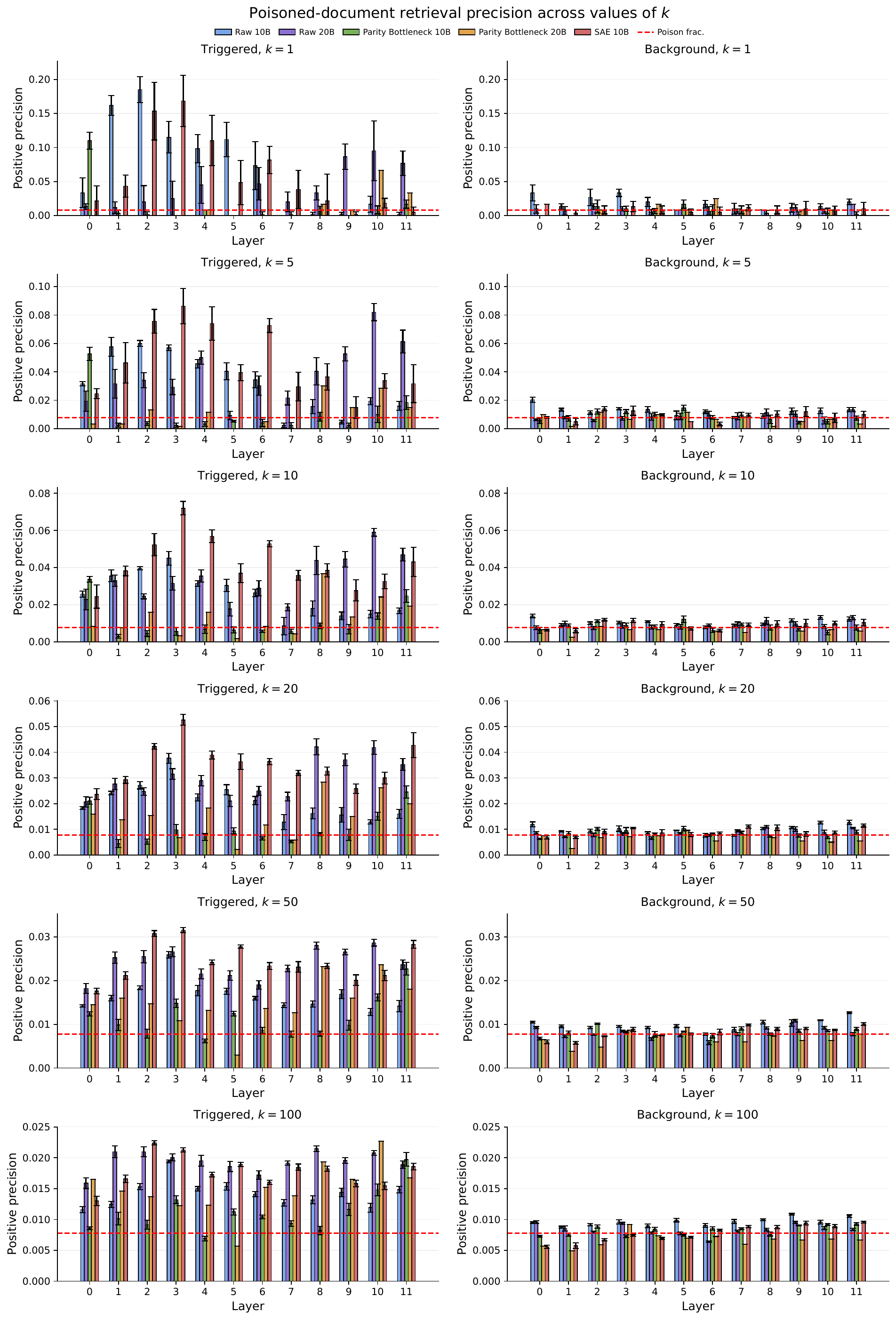}
    \caption{
    Poisoned-document retrieval precision for additional values of $k$.
    We report triggered-query and background-query precision for raw activations, Parity bottleneck features, and SAE features.
    Bars show mean precision over five repeated runs, and error bars denote one standard deviation.
    }
    \label{fig:retrieval_all_k}
\end{figure*}

\section{Fine grained causal intervention}
\label{app:fine grained CI}

\subsection{Task Descriptions}
We evaluate sparse controllability on three tasks, each targeting a different circuit type.

\textbf{Indirect Object Identification (IOI).} Prompts follow the template ``When [IO] and [S] went to the [place], [S] gave a [object] to'', where IO is the indirect object (non-repeated name) and S is the subject (repeated name). The model must predict IO over S at the final token position. Edit pairs are constructed by replacing the IO name while keeping S, position pattern, and template fixed. Success is measured as logit(target IO) $>$ logit(S) after intervention. We generate 500 prompts with balanced name/position combinations and construct 200 edit pairs.

\textbf{Gendered Pronouns.} Prompts follow the template ``The [occupation] said that'', where the model predicts a gendered pronoun (he/she) based on stereotypical occupation-gender associations. Edit pairs swap between female-stereotyped occupations (nurse, teacher, librarian) and male-stereotyped occupations (doctor, engineer, mechanic). Success is measured as logit(target pronoun) $>$ logit(source pronoun) after intervention. We use 100 pairs after filtering to those where the source correctly predicts its expected pronoun.

\textbf{Greater-Than.} Prompts follow the template ``The war lasted from 17XX to 17'', where the model must predict a digit greater than XX/10. Edit pairs swap between different source years, changing which digit range is expected. Success is measured as logit(target digit) $>$ logit(source digit) after intervention. Due to the model's 50\% capability on this task, we filter to pairs where the source prompt correctly predicts, yielding 17--23 valid pairs per run.

\subsection{MLP selection via Causal Mediation}
For each task and model, we identify the evaluation MLP through activation patching with a logit-difference metric. For each layer $\ell$ and component (attention output, MLP output, or full block output), we:

\begin{enumerate}
    \item Compute the clean logit difference $d_{\text{clean}} = \text{logit}(a_{\text{src}}) - \text{logit}(a_{\text{tgt}})$ on the source prompt.
    \item Compute the target logit difference $d_{\text{tgt}}$ on the target prompt.
    \item Patch the component output at layer $\ell$ from the target prompt into the source prompt at \emph{all token positions}.
    \item Measure the patched logit difference $d_{\text{patched}}$.
    \item Compute the normalized effect: $e = (d_{\text{clean}} - d_{\text{patched}}) / (d_{\text{clean}} - d_{\text{tgt}})$.
\end{enumerate}

A value of $e = 0$ indicates the patch had no effect; $e = 1$ indicates the logit difference fully shifted to match the target. We patch all positions simultaneously rather than only the final token, as MLP-mediated circuits may process task-relevant information at subject or object token positions.

For the parity bottleneck model, we select the layer with the highest MLP logit-diff effect for gendered pronouns (layer 6, $e = 0.15$) and greater-than (layer 10). For IOI, which is attention-dominated, we select the layer corresponding to the name mover attention heads (layer 11, attention $e = 0.42$). For the dense baseline, we independently identify each task's optimal layer via the same procedure and evaluate there: layer 11 for IOI (attention $e = 0.38$), layer 1 for gendered pronouns (MLP $e = 0.94$), and layer 1 for greater-than (MLP $e = 0.95$).

\subsection{Greedy Feature Editing Algorithm}

Given a source prompt $p_s$ and target prompt $p_t$, we seek to edit the minimum number of SAE features to flip the model's prediction. The optimization proceeds on a \emph{per-prompt} basis, selecting different features depending on what is active for each specific input. This follows the methodology of \citet{makelov2024towards}, who note that per-prompt optimization avoids disadvantaging dictionaries that do not dedicate a uniform set of features to each attribute.

Let $\mathbf{a}_s$ and $\mathbf{a}_t$ denote the source and target activations at the evaluation layer. We encode both through the SAE and perform $k$ greedy edit steps, where each step selects the single feature modification that minimally reduces $\|\hat{\mathbf{a}}_{\text{edited}} - \mathbf{a}_t\|_2^2$.

\textbf{For the parity bottleneck:} We hook inside the DPB module to capture the per-level feature indices and activation scores directly. At each greedy step, for every feature position $(l, i)$ across all levels, we consider three operations:
\begin{itemize}
    \item \emph{Zero}: set the activation score at position $(l, i)$ to zero.
    \item \emph{Replace}: substitute both the feature index and score at position $(l, i)$ with the target's values at the same position. This is valid because features at the same position belong to the same hierarchical level and thus the same index space.
    \item \emph{Rescale}: keep the source feature index but adopt the target's activation score.
\end{itemize}
After selecting the best edit, we decode via the standard DPB reconstruction path with norm preservation using the source input's norm, ensuring the output magnitude is consistent.

\textbf{For post-hoc SAEs (TopK, BatchTopK, JumpReLU, Matryoshka):} We encode the source activation into the $d_{\text{sae}}$-dimensional feature space. At each greedy step, for every feature index $j$ that is active in either the source or target encoding, we set the source's activation at index $j$ to the target's value and evaluate the resulting reconstruction distance. The edit that produces the smallest $\ell_2$ distance to the target activation is selected.

In both cases, the greedy search evaluates $O(k_{\text{active}})$ candidates per step, where $k_{\text{active}}$ is the number of active features (48 for all methods). The total cost is $O(k \cdot k_{\text{active}})$ SAE decode operations per prompt pair.

\subsection{Patching and Success Measurement}

After editing $k$ features and decoding back to activation space, we patch the edited activation into the model at the block output (residual stream) of the evaluation layer at the final token position. This replaces the residual stream value that would normally propagate to subsequent layers.

We measure success as a binary outcome: does the model now predict the target answer over the source answer? Specifically, for IOI we check logit(target IO name) $>$ logit(S name); for gendered pronouns we check logit(target pronoun) $>$ logit(source pronoun); for greater-than we check logit(target digit) $>$ logit(source digit).

\textbf{Pair filtering.} To ensure the success rate measures genuine prediction flips rather than noise from already-incorrect predictions, we filter to only include pairs where the source prompt correctly predicts the source answer before any intervention. This typically retains 90--98\% of pairs for IOI and gendered pronouns, and 57--77\% for greater-than (reflecting its lower baseline capability).

\subsection{Fairness Considerations}

To ensure a fair comparison between in-loop and post-hoc approaches:

\begin{itemize}
    \item \textbf{Layer selection}: Each model is evaluated at its own optimal layer as determined by causal mediation. The parity bottleneck is evaluated where its features participate in the circuit; post-hoc SAEs are evaluated where their base model concentrates the relevant computation.
    \item \textbf{Feature budget}: All methods use $k = 48$ active features (matching the parity bottleneck's total active feature count), and we evaluate at the same values of edits ($k = 1, 2, 4$).
    \item \textbf{Pair filtering}: The same filtered pair set is used for all methods on each task, ensuring identical evaluation conditions.
    \item \textbf{Multiple SAE architectures}: We evaluate four post-hoc SAE variants (TopK, BatchTopK, JumpReLU, Matryoshka BatchTopK) to demonstrate that the gap is not specific to one SAE training method but reflects a fundamental limitation of post-hoc approaches.
    \item \textbf{Per-prompt optimization}: The greedy search operates independently per prompt, avoiding assumptions about which features should universally correspond to task attributes. This does not disadvantage post-hoc methods that may use different features for the same concept across contexts.
\end{itemize}

\section{Steering}
\label{app:steering}

In initial experiments, we found that different steering strategies provided better results for dense+post-hoc SAE vs DPB features. In particular, for SAEs, single-layer steering provides a better tradeoff between fluency and concept activation, while for DPB, multi-layer steering works better. We thus compare the two using the best available method for each approach, multi-layer for DPB, and single-layer for SAE. For SAEs we steer at the middle layer of the model (layer 12).

Multi layer steering exploits the parity bottleneck's structural advantage: perturbations at each layer are denoised by subsequent layers' bottlenecks. This allows small, distributed nudges to compound coherently without accumulating off-manifold noise. The effective alpha range is much lower (2--10) compared to single-feature steering (20--75) because perturbations are applied at every layer simultaneously.

The same concept definitions, seed texts, LLM judge, and scoring rubric are used for both models. The scaling factor $\alpha$ is swept to trace the concept-fluency tradeoff, but differ between PT and baseline due to the aforementioned reasons. 

\subsection{Steering Mechanisms}

\textbf{Obtaining steering vector} 
Both models use the same contrastive procedure to obtain steering vectors:

\begin{enumerate}
    \item Collect positive texts (containing the target concept) and negative (contrastive) texts.
    \item Run the model, record feature activations on each set.
    \item Compute the difference in mean feature activations between positive and negative sets.
    \item Retain the top-$k$ features with the largest positive differential activation.
\end{enumerate}

\textbf{Steering vector application} For post-hoc SAEs, steering vectors $\mathbf{v}$ are top k weighted sums of the  the decoder column $W_{\text{dec}}[i]$ normalized to unit norm. For Parity Transformer with DPB, L0 features are standard basis one-hot vectors, while L1+ features are the parity-basis vector. For the baseline, steering vector $\mathbf{v}$ is added to the residual stream at the SAE's hook point during generation: $\mathbf{x} \leftarrow \mathbf{x} + \alpha \cdot \mathbf{v}$. For Parity Transformer, since the decoder matrix is not materialized, steering is done by editing the active features and/or their activations. 

\subsection{Concept Definition and Feature Selection}

For both models, we define 100 target concepts of simple semantic categories (e.g., ``religion,'' ``space,'' ``war'') generated by an GPT4o, rather than drawing from Neuronpedia concept databases like in \citet{luo2026learning}. This ensures neither model is advantaged by concept granularity mismatched to its capacity: at the 1.3B parameter scale, features tend to be broader than the fine-grained concepts in databases designed for larger models.

\subsection{Evaluation Protocol}

For each concept and steering strength $\alpha$, we generate text from 30 diverse seed prefixes drawn from AG News (covering business, technology, sports, and science domains). Each generated text is scored by an LLM judge (GPT-4o-mini) on two dimensions:

\begin{itemize}
    \item \textbf{Concept relevance} (0--2): To what extent did the generated text shift toward the target concept compared to unsteered generation from the same seed?
    \item \textbf{Fluency} (0--2): Is the text grammatically correct and naturally flowing?
\end{itemize}

Scores are averaged across seeds for each (concept, $\alpha$) pair, yielding a single (fluency, concept) point. Sweeping over $\alpha$ traces a Pareto frontier:

\textbf{Judge prompt for Concept score}
                                                                                                            
  Below is a BASELINE text (normal model output) and a STEERED text (output after concept injection). Both continue from the same prompt.                                                            
                                                                            
  Your job is to rate whether the STEERED text shows a coherent semantic shift that includes the target concept compared to the BASELINE. The steered text may also shift toward related or          
  co-occurring concepts — this is acceptable and should not be penalized as long as the target concept is present. Only penalize if the shift is toward concepts clearly antithetical to or unrelated
   to the target.                                                                                                                                                                              
  Rate the concept presence on a scale from 0 to 2:                                                                     
  - 0: the steered text shows NO shift toward the concept compared to baseline, or shifts in an antithetical direction
  - 0.5: very subtle hints or slight vocabulary shifts toward the concept                                                                                                               
  - 1: moderate presence of concept-related words, themes, or references (co-occurring related concepts are fine)                                                                 
  - 1.5: strong presence of the concept throughout the text                                                                                                                             
  - 2: the text is dominated by the concept  

\textbf{Judge prompt for fluency}

Please act as an impartial judge and evaluate the fluency of the sentence fragment provided below. Focus solely on fluency, disregarding its completeness, relevance, coherence with any broader context, or informativeness.                                                                                    
 Begin your evaluation by briefly describing the fluency of the sentence, noting any unnatural phrasing, awkward transitions, grammatical errors, or repetitive structures that may hinder      
 readability. After providing your explanation, rate the sentence's fluency on a scale from 0 to 2, where 0 indicates the sentence is not fluent and highly unnatural (e.g., incomprehensible or
 repetitive), 1 indicates it is somewhat fluent but contains noticeable errors or awkward phrasing, and 2 indicates the sentence is fluent and almost perfect.

\newpage
\section{Example steered completions}\label{app:steered completions}
\begin{figure}[h]
        \includegraphics[width=\textwidth]{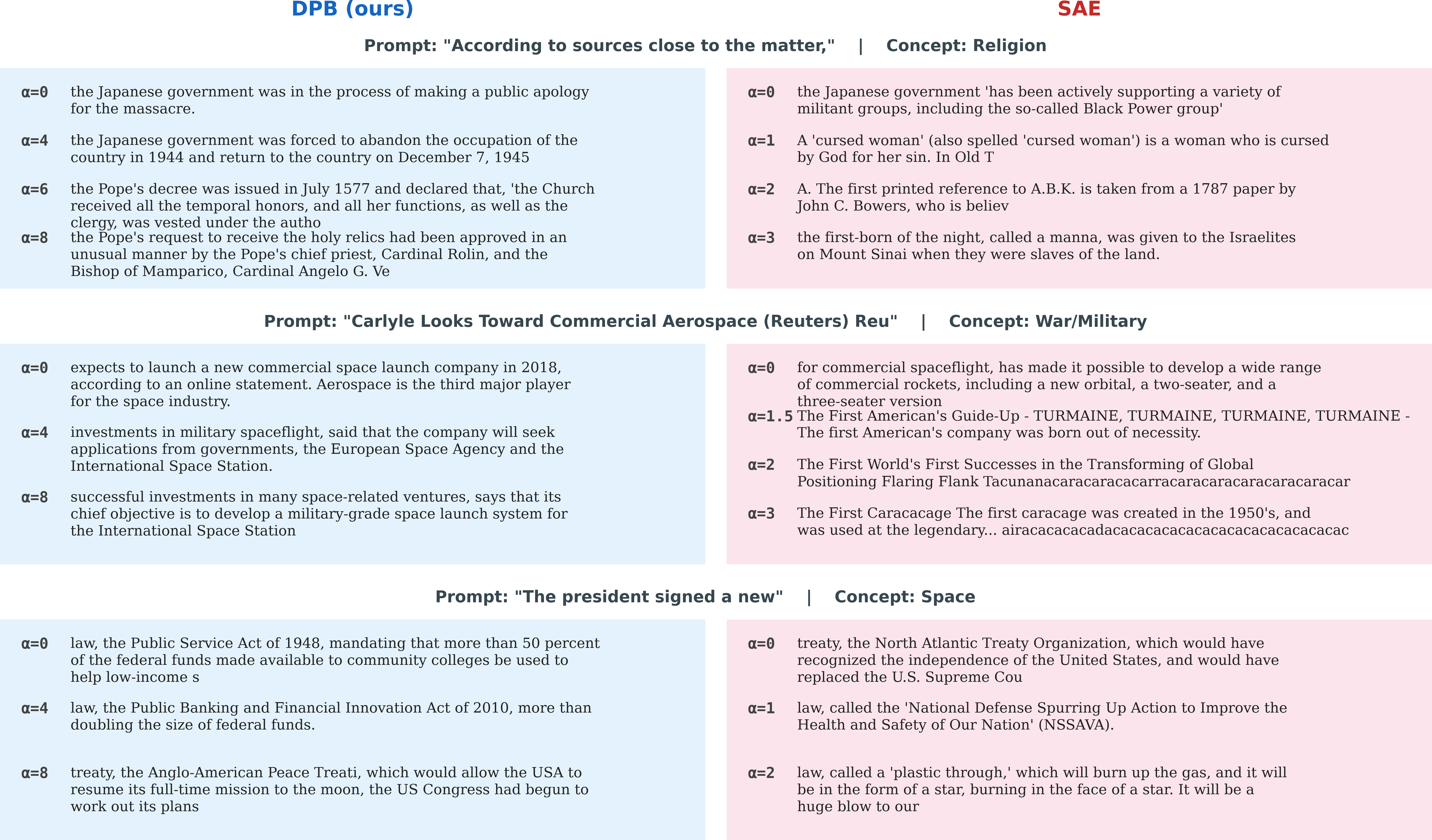}
        \caption{Steering examples}
        \label{fig:steering-examples}
\end{figure}

\section{Compute Requirements}
Training a single 1.3B ParityTransformer on 20B tokens takes around 160 B200 hours, for an approximate cost of \$640, while a 200m ParityTransformer takes around 32 B200 hours, for an approximate cost of \$128. We trained 2 variants each for each size for a total of \$1,536 to train the ParityTransformer models. Training the dense baselines took under \$300. Additionally, the data-poisoning experiments cost around 16 B200 hours, for a total of \$64. Training SAE baselines took around 64 B200 hours for a total of \$256. The total cost to reproduce all experiments is around  \$2,200. We estimate that we spent less than \$15,000 including preliminary experiments.

\newpage
\section{Model configuration reference}\label{app:DPB configs}
Training details for each model are given in \cref{tab:model-inventory}.

At the 1.3\,B (Large) scale we report two dense baseline models:
\textbf{GPT-Large} is used for the performance evaluations in
\S\ref{sec:variants};
\textbf{GPT-Large-cold} is the dense baseline reference for the data efficiency estimate \S\ref{sec:tax} and
interpretability experiments elsewhere in the paper.
The PT model evaluated in interpretability experiments
(\S\ref{sec:steering}, \S\ref{sec:saebench}, Appendix~\ref{app:fine grained CI})
is \textbf{PT-Large-2L-Aux}; \textbf{PT-Small-2L} is used for the
poisoned-document retrieval experiments (\S\ref{sec:poisoned-data}).
PT-Large-2L, PT-Small-3L, and the dense baselines appear only in the
performance evaluation (\S\ref{sec:variants}).

PT-Large-2L-Aux uses both the variance and reconstruction auxiliary losses, with weight $.05$ on variance and $.1$ on reconstruction. All other models use no auxiliary loss.

All PT models are trained trained for 20\,B tokens from FineWeb-Edu with sequence length 1024. The dense baselines GPT-Large and GPT-small are stopped early once reaching $1\%$ of the validation loss of PT-Large-2L and PT-Small-3L respectively, while using linear learning rate schedules with planned token budget and warmdown fraction given in \cref{tab:model-inventory} as ``$w_f$''.  Peak learning rate is shown
  as ``Muon\,/\,AdamW'' when the two are decoupled.

To provide as fair an estimate of data efficiency for the large dense model as possible, GPT-Large-cold was trained using a token budget chosen to achieve the target validation loss while exhausting the learning rate schedule, with the specific token budget of 3B tokens determined based on extrapolation from initial hyper-parameter sweeps.

\begin{table}[h]
  \caption{Canonical models used in the paper.}
  \label{tab:model-inventory}
  \centering
  \small
  \setlength{\tabcolsep}{4pt}
  \begin{tabular}{lrrrrrrlr}
    \toprule
    Label & Params & Layers & $d$ & Levels & Peak LR & $w_f$ & \makecell{Planned Budget \\ (billion tokens)} & 
    \makecell{Stopped Budget \\ (billion tokens)} \\
    \midrule
    \multicolumn{9}{l}{\emph{Large (1.3B)}} \\
    PT-Large-2L     & 1.31B & 24 & 2048 & 2   & 8e-3 / 3e-4 & 0.5 & 20 & 20 \\
    PT-Large-2L-Aux   & 1.31B & 24 & 2048 & 2   & 8e-3 / 3e-4 & 0.5 & 20 & 20  \\
    GPT-Large       & 1.31B & 24 & 2048 & --- & 8e-3 / 3e-4 & 0.5 & 26 & 6.82\\
    GPT-Large-cold  & 1.31B & 24 & 2048 & --- & 1e-2        & 0.5 & 3 & 3 \\
    \midrule
    \multicolumn{9}{l}{\emph{Small (203M)}} \\
    PT-Small-3L     & 203M  & 12 & 1024 & 3   & 1.2e-2      & 0.9 & 20 & 20 \\
    PT-Small-2L     & 203M  & 12 & 1024 & 2   & 1.2e-2      & 0.9 & 20 & 20 \\
    GPT-Small       & 205M  & 12 & 1024 & --- & 6e-3        & 0.4  & 10 & 2.13 \\
    \bottomrule
  \end{tabular}
\end{table}

ParityTransformer hierarchy configurations are detailed in \cref{tab:cayley-levels}. For a
  level $\ell$, $n_\ell$ is the bit-width of the feature index ($2^{n_\ell}$
  features at that level), $K_\ell$ is the number of features active per
  token, and $\Delta_\ell$ is the per-parent budget --- the number of
  children scored per active parent at the previous level. $\Delta_0$
  is undefined since L0 has no parents. The L0 ($\ell{=}0$) features are
  basis-aligned with the residual stream; L1+ features are the parity
  hashes described in \S\ref{sec:methods}.

\begin{table}[h]
  \caption{Per-level dictionary configuration for the
  ParityTransformer variants in Table~\ref{tab:model-inventory}. }
  \label{tab:cayley-levels}
  \centering
  \small
  \setlength{\tabcolsep}{5pt}
  \begin{tabular}{lcccc}
    \toprule
    Variant & Level $\ell$ & $n_\ell$ (features $=2^{n_\ell}$) & $K_\ell$ & $\Delta_\ell$ \\
    \midrule
    \multirow{2}{*}{\shortstack[l]{2L variants: PT-Large-2L, PT-Large-2L-Aux, PT-Small-2L}}
        & 0 & $\log_2 d$ ($d{=}1024$ or $2048$) & 16 & --- \\
        & 1 & 15 (32{,}768)                    & 32 & 256 \\
    \midrule
    \multirow{3}{*}{\shortstack[l]{3L variant: PT-Small-3L}}
        & 0 & 10 (1024)     & 16 & --- \\
        & 1 & 15 (32{,}768) & 32 & 256 \\
        & 2 & 17 (131{,}072)& 64 & 256 \\
    \bottomrule
  \end{tabular}
\end{table}


\end{document}